\documentclass[12pt]{article}

\usepackage{macros}

\usepackage{xcolor}

\usepackage{amsopn}

\usepackage{lmodern}
\usepackage{amsmath}

\usepackage{graphicx}

\usepackage{amssymb,amsmath}

\usepackage{relsize}

\usepackage{upgreek}

\usepackage{subfigure}

\usepackage{url}

\usepackage{graphicx}

\usepackage{algorithmic}
\usepackage{algorithm2e}

\pdfoutput=1
\usepackage{amsmath, amssymb}
\usepackage{graphicx,psfrag,epsf}
\usepackage{enumerate}
\usepackage{natbib}
\usepackage{url} 

\DeclareMathOperator{\argmin}{argmin}
\def\calX{\mathcal{X}}
\def\U{\mathbf{U}}

\def\V{\mathbf{V}}
\def\I{\mathbf{I}}
\def\D{\mathbf{D}}

\def\R{\mathbb{R}}
\def\x{\mathbf{x}}

\def\X{\mathbf{X}}

\def\w{\mathbf{w}}

\def\z{\mathbf{z}}
\def\btheta{\boldsymbol{\theta}}
\def\v{\mathbf{v}}

\def\0{\mathbf{0}}
\def\u{\mathbf{u}}
\def\fupper{\overline f}
\def\flower{\underline f}

\def\Fdx{F_{(\x)}}

\def\W{\mathbf{W}}

\newcommand{\nnn}[2]{X^{({#2})}_{#1}}

\newcommand{\nnnd}[2]{D_{{#1},({#2})}}

\newcommand{\nnnv}[2]{\mathbf{V}_{{#1},({#2})}}

\begin{document}

\def\spacingset#1{\renewcommand{\baselinestretch}%
{#1}\small\normalsize} \spacingset{1}


  \title{\bf Efficient Estimation of High Information Projections using Nearest Neighbours}
  \author{David Hofmeyr
  \hspace{.2cm}\\
    School of Mathematical Sciences, Lancaster University}
  \maketitle

\bigskip
\begin{abstract}
An intuitive method for dimensionality reduction is proposed, which is highly effective for finding interesting projections of multivariate data. Following similar intuitive motivation to a number of existing techniques, the proposed method is based on enhancing the nearest neighbour relationships in the data. The proposed projection arises from the spectral decomposition of a matrix designed to encode the local covariance structure in the data, where the local covariance at a point is captured by pairs of its nearest neighbours. We show that under standard regularity conditions this matrix is a consistent estimator of the so-called ``Density Information Matrix'' (DIM); a non-parametric analogue of the Fisher Information Matrix. Spectral decompositions of DIMs have been shown to be connected with the important problems of Independent Components Analysis and, in the supervised context, Sufficient Dimension Reduction. However, existing estimators of the DIM are computationally expensive to compute and only target the DIM of a surrogate density, which is proportional to the square of the true underlying density. 
In addition, we go on to explore the practical utility of our method in aiding the downstream tasks of cluster analysis and outlier detection.
\end{abstract}

\noindent%
\spacingset{1.45} 
\section{Introduction}
\label{sec:intro}

At a high level the task of dimension reduction is to extract from a set of multivariate observations, a reduced representation which retains as well as possible the useful information/structure in the data. Structure may be represented at a global level, as in the ubiquitous Principal Components Analysis and its variants~\citep[PCA]{pca}; or at a more local level, e.g. as in those methods based on manifold learning, which typically rely on a graph-based representation of the data~\citep{lpp}. Another perspective on the problem is to operate at the level of the probability distribution of the reduced representation, where (after accounting for scale) low entropy distributions are seen to contain more interesting structure than high entropy ones~\citep{entropyPP}, with both multimodality and long-tailedness being associated with low entropy. 

In this paper we introduce a simple and intuitive linear dimensionality reduction technique which is based on enhancing the local structure in a set of data. Although similar in some respects to those based on a nearest neighbour graph constructed from the data, we take our motivation more from the information/entropy perspective. In particular, let $\calX := \{\x_1, ..., \x_n\}$ be a set of obervations in $\R^p$ and for $k < n$ let $\x^{(k)}_i$ be the $k$-th nearest neighbour of $\x_i; i \in [n]$~\footnote{We have used $[n]$ to be the first $n$ natural numbers, i.e. $[n] = \{1, ..., n\}$}. Then define
\begin{align}
    I(\calX, k) = \frac{p}{2n}\sum_{i=1}^n\left(\left(\frac{1}{\nnnd{\x_i}{k}^{2}} + \frac{1}{\nnnd{\x_i}{k+1}^{2}}\right)\I - \frac{2p}{\nnnd{\x_i}{k}^2\nnnd{\x_i}{k+1}^2}\nnnv{\x_i}{k}\right),
\end{align}
where $\nnnd{\x_i}{k} := ||\x_i - \x_i^{(k)}||$ is the $k$-th neighbour distance of $\x_i$ and $\nnnv{\x_i}{k} = \frac{1}{2}(\x_i^{(k)}-\x_i^{(k+1)})(\x_i^{(k)}-\x_i^{(k+1)})'$ represents a coarse estimate of the local covariance at $\x_i$. At its essence our approach may be seen from the point of view of projecting $\calX$ on the leading eigenvectors of $I(\calX, k)\Sigma_\calX$, where $\Sigma_\calX$ is the covariance of $\calX$. Although we make a few adjustments to this, which we describe in detail in Section~\ref{sec:practicalities}\footnote{We deviate from this by (i) averaging $I(\calX, k)$ over multiple values of $k$ to reduce variance; (ii) thresholding the eigenvalues based on theoretical asymptotic lower bounds; and (iii) orthogonalising the projection for better practical utility on downstream tasks.}, it is worthwhile first considering this formulation as it offers immediate intuitive motivation.

%
%
Notice that $I(\calX, k)$ may be expressed as $\alpha \I - \sum_{i=1}^n \beta_i \nnnv{\x_i}{k}$ for some $\alpha, \beta_i; i \in [n]$. The leading eigenvectors of $I(\calX, k)\Sigma_\calX$ may therefore be seen to represent projections which lead to tighter/more compact local covariance relative to the overall covariance of the projected data. Moreover, notice that the coefficients $\beta_i \propto 1/\nnnd{\x_i}{k}^2\nnnd{\x_i}{k+1}^2; i \in [n],$ will emphasise this compactness around points whose neighbour distances are smaller, or equivalently those at which the density of points is generally higher. In other words, projecting onto the leading eigenvectors of $I(\calX, k)\Sigma_\calX$ will tend to lead to especially compact/high density regions, relative to the overall scale of the projected data, and in such a way which emphasises the already higher density points, leading to ``peaky'' (low entropy) distributions. In addition, as we show in the remainder, $I(\calX, k)$ is a consistent estimator of the Density Information Matrix~\citep[DIM]{DIM}; a nonparametric analogue of the Fisher Information Matrix. Spectral decompositions of DIMs have been connected with the task of Independent Components Analysis (ICA), which further links the proposed approach with entropy minimisation since ICA may be formulated as a minimum entropy task~\citep{entropyPP}.

The DIM possesses some remarkable properties at a population level, but is very challenging to estimate. Indeed, as far as we are aware our proposal is the first direct estimator of the DIM, where existing approaches only target the DIM of a surrogate density which is proportional to the square of the data density~\citep{DIM0, DIM, CIM}. Although it has been shown that the spectral decompositions of the DIMs from both the true and surrogate densities are similar in many respects, we have found far better practical utility in the proposed formulation, especially on high dimensional applications. We show two examples in Figure~\ref{fig:tasters}. 
%
The first is a data set commonly used to assess the performance of clustering models; the Multiple Features Database~\citep{mfdigits}. Combining all five associated data sets, these data comprise a total of 649 features derived from 2000 images of handwritten digits taken from Dutch utility maps. The second is the Election2005 data set, taken from the {\tt R} package {\tt OutliersO3}~\citep{OutliersO3}. The data correspond with details from two German elections, from 2002 and 2005. We used all 66 numerical features, and have highlighted in the plots the ten points most frequently identified as outliers as described in the package vignette\footnote{\url{https://cran.r-project.org/web/packages/OutliersO3/vignettes/DrawingO3plots.html}}.


\begin{figure}
    \centering
    \subfigure[Multiple Features Database]{\includegraphics[width=\linewidth]{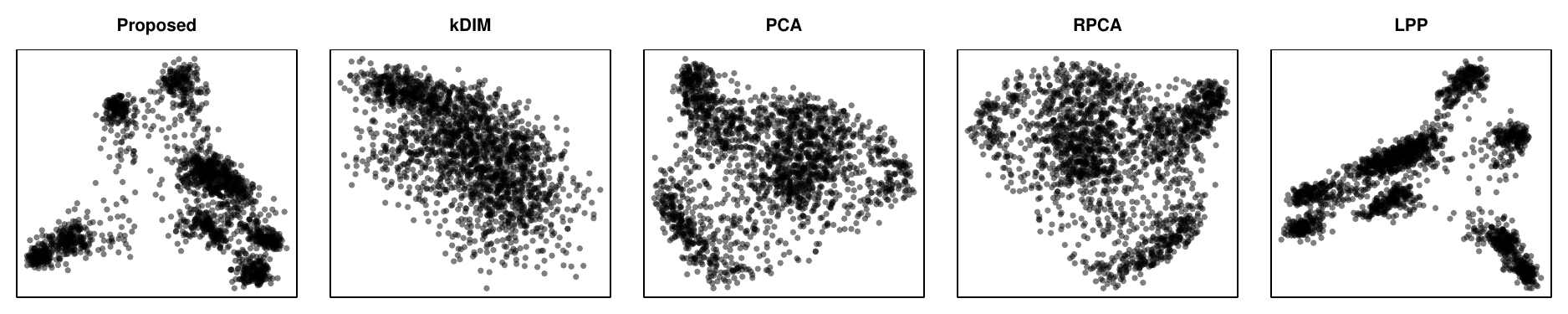}\label{fig:tasters1}}
    \subfigure[Election2005]{\includegraphics[width=\linewidth]{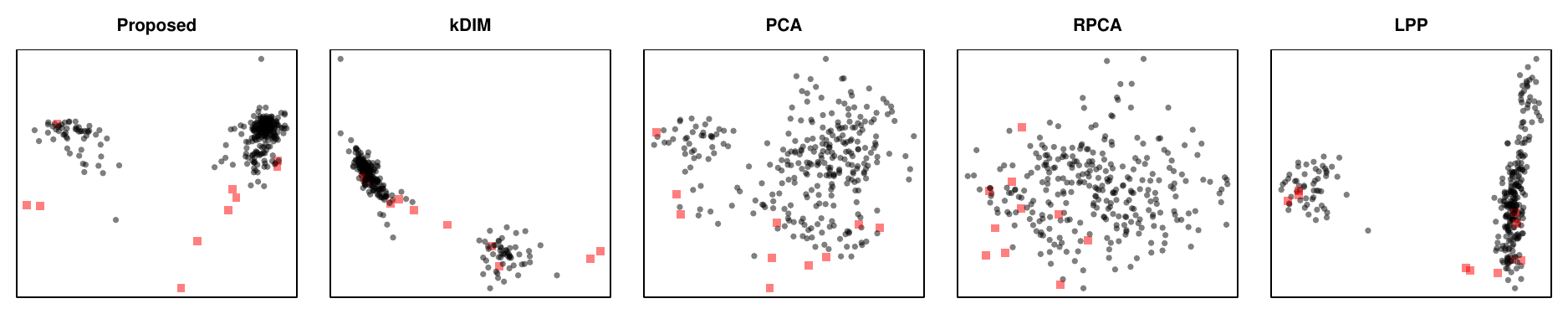}\label{fig:tasters2}}
    \caption{Two dimensional projections of (a) the Multiple Features Database; and (b) the Election2005 data set with potential outliers highlighted.}
    \label{fig:tasters}
\end{figure}

In each case the plots show two dimensional projections of the data from (i) the proposed approach; (ii) the method of~\cite{DIM}, which projects the data onto the eigenvectors of $I(\hat f^*)\Sigma_{\cal X}$, where $I(\hat f^*)$ is a kernel based estimate of the DIM of the surrogate density mentioned above; (iii) PCA; (iv) a robust variant of PCA~\citep[RPCA]{robpca}; and (v) Locality Preserving Projections~\citep[LPP]{lpp}.
PCA is arguably the most well recognised dimension reduction technique and is based on finding an orthogonal projection which maximises the variance of the projected data. The robust variant used here instead finds the projection which maximises the variance of the projections of the least outlying points. Finally, LPP is a well recognised graph based method, and is a linear approximation of the Laplacian eigenmap.


In the case of the Multiple Features data set, Figure~\ref{fig:tasters1}, both LPP and the proposed approach show pleasing separation of numerous clearly defined clusters. PCA shows some evidence of the presence of clusters but with much less clear separation, and the projection found with RPCA is almost identical up to a reflection in the vertical axis. The DIM based on the method of~\cite{DIM} shows very little structure by comparison. 

On the other hand, when applied to the Election2005 data set, Figure~\ref{fig:tasters2}, all methods place multiple of the potential outliers beyond or near the periphery of the higher density regions, but only the DIM based approaches (the proposed and kDIM methods) separate some of these well away from the remaining data and arguably the proposed approach achieves this more clearly than kDIM.\\
\\
The remainder of the paper is organised as follows. In the next section we briefly review some of the theoretical properties of the DIM at a population level, while in Section~\ref{sec:theory} we explore the theoretical properties of $I(\calX, k)$ as an estimator for the DIM. In Section~\ref{sec:practicalities} we describe some of the practicalities associated with our methodology, and we go on, in Section~\ref{sec:experiments}, to further document its practical utility in experiments using popular data sets in the public domain. We then conclude the paper with a brief discussion, in Section~\ref{sec:conclusions}.



\section{The Density Information Matrix}\label{sec:dim}

The Density Information Matrix (DIM), described by~\cite{DIM0, DIM}; but only named as the DIM later on~\citep{CIM}, may be seen as a non-parametric analogue of the Fisher Information Matrix (FIM), well known for its relevance in the theory of maximum likelihood estimation. If $X$ is a random variable on $\R^p$, with density $f_X$ parameterised by a $\btheta \in \R^q$, then the FIM is defined as
\begin{align}
    I(\btheta) &:= -E\left[\nabla^2_{\btheta} \log\left(f_X(X|\btheta)\right)\right] = E\left[\frac{1}{f_X(X|\btheta)^2}\nabla_{\btheta} f_X(X|\btheta)\nabla_{\btheta} f_X(X|\btheta)'\right],\label{eq:FIM}
\end{align}
with the equality above holding under the standard regularity conditions of maximum likelihood theory. Here we have used $\nabla_{\btheta}$ to be the gradient operator (with respect to $\btheta$), i.e., $\nabla_{\btheta}f_X(\x|\btheta) = (\frac{\partial}{\partial \theta_1}f_X(\x|\btheta), ..., \frac{\partial}{\partial \theta_q}f_X(\x|\btheta))$, and similarly $\nabla^2_{\btheta}$ the corresponding Hessian operator, i.e., the matrix of second partial derivatives; $\nabla^2_{\btheta}f_X(\x|\btheta)_{i,j} = \frac{\partial^2}{\partial \theta_i\partial\theta_j}f_X(\x|\btheta)$. At an intuitive level the FIM may be interpreted as the average amount of information a realisation of $X$ carries about the parameter $\btheta$, and hence about its distribution %
(since its density is characterised by the value of $\btheta$). This offers pleasing interpretation in the context of maximum likelihood estimation, where $I(\btheta)^{-1}$ represents the (asymptotic) variance of the maximum likelihood estimator for $\btheta$; and any natural interpretation of $X$ being ``more informative about $\btheta$'' ought to be consistent with a ``standard'' estimator for $\btheta$ having lower variance.

In the non-parametric context, where no parametric form of $f_X$ is assumed, the DIM is defined as
\begin{align}
    I(f_X) := E\left[\frac{1}{f_X(X)^2}\nabla f_X(X)\nabla f_X(X)'\right],\label{eq:DIM}
\end{align}
where here the gradient is taken with respect to the argument of $f_X$, i.e., $\nabla f_X(\x) = \left(\frac{\partial}{\partial x_1}f_X(\x), ..., \frac{\partial}{\partial x_p}f_X(\x)\right)$. Transferring the intuitive interpretation of the FIM to the non-parametric context, the DIM may be seen as representing the average information which a realisation of $X$ carries about its density function. Minimal thought can thus lead us to an intuitive connection between the elements in $I(f_X)$ and the dependence structure in the components of $X$. Specifically, if $X_i$ provides information about the distribution of $X_j$, where $i, j \in [p]$, then $X_i$ and $X_j$ are not independent. Indeed it has been shown that $I(f_X)_{i,j} \not = 0$ implies that $X_i$ and $X_j$ are not independent~\citep{DIM}. In fact the properties of $I(f_X)$, in relation to the dependence structure in the components of $X$, go further and include potential diagnostics for conditional independence among groups of components~\citep{DIM}. As such the DIM acts similarly to the precision matrix of a multivariate Gaussian density, and it is straightforward to show that if $X$ has a multivariate Gaussian distribution then its DIM and precision matrix are equal.\\
\\
It has also been shown that, after accounting for scale, the Gaussian densities have the minimum information, and in particular $I(f_X) \succcurlyeq \Sigma_X^{-1}$,
with equality holding if and only if $X$ has a Gaussian distribution~\citep{DIM0}. From the point of view of dimension reduction, then, projecting $X$ onto the leading eigenvectors of $I(f_X)\Sigma_X$ may be seen to maximise departure from Gaussianity. In fact the DIM was first used for the dual problem of finding so-called ``white noise subspaces'', i.e. subspaces within which the projection of $X$ has a Gaussian distribution, by considering the space spanned by the eigenvectors of $I(f_X)\Sigma_X$ associated with eigenvalues equal to one~\citep{DIM0}. Retaining only the projection of $X$ onto the orthogonal complement of this white noise subspace is thus seen to retain potentially interesting structure in the distribution of $X$.

This perspective also connects the DIM to the problem of ICA, since it is also known that, again after accounting for scale, the Gaussian densities have the greatest entropy. Indeed numerous methods for performing ICA focus on the problem of maximising ``non-Gaussianity'' instead of the more challenging problem of directly minimising entropy~\citep{fastICA}. Perhaps unsurprisingly, then, it was also shown that if $X$ satisfies the conditions of the independent components model~\citep{comonICA}, i.e., that $X = \mathbf{M}S$, where $S$ is a random variable on $\R^p$ whose components are mutually independent and $\mathbf{M} \in \R^{p\times p}$ is a non-singular matrix, then the spectral decomposition of $I(f_X)\Sigma_X$ recovers the independent components up to unknown scaling and permutation~\citep{DIM}.




\section{$I(\calX, k)$ as an Estimator of the DIM}\label{sec:theory}

In this section we explore the theoretical properties of $I(\calX, k)$, in relation to the DIM of the density underlying the data. To that end we consider a sequence of independent random variables $X_1, X_2, ...$, identically distributed to random variable $X$, with continuous distribution function $F_X$ admitting density function $f_X$. Then for $n \in \mathbb{N}$ we let $\X_n := \{X_1, ..., X_n\}$ be the set containing the first $n$ elements of the sequence, and for $i \in [n]; k \in [n-1]$ let $X_{i,n}^{(k)}$ be the $k$-th nearest neighbour of $X_i$ from among $\X_n$, and let $\nnnd{i,n}{k} = ||X_i-X_{i,n}^{(k)}||$ be the corresponding $k$-th neighbour distance. Our main result is then stated in the following theorem, where we have used the notation $B_h(\x) = \{\z \in \R^p | ||\x-\z||\leq h\}$ to be the ball of radius $h$ centered on $\x$, and for a smooth measureable set $V \subset \R^p$ we use $\oint_{V} g(\x)d\x$ to indicate the surface integral of the function $g$ over $\partial V$, the surface of $V$. We also use $a_0$ and $v_0$ to be the surface area and volume of the unit ball, i.e. $a_0 = \oint_{B_1(\0)}1 d\x$ and $v_0 = \int_{B_1(\0)}1 d\x$, in $\R^p$. In addition we use $\Fdx$ to denote the distribution function of the random variable $||X-\x||$.

\begin{theorem}\label{thm:conv}
    Let $X_1, X_2, ...$ be a sequence of i.i.d. random variables on $\R^p; p > 4$, with density function $f_X$ satisfying
    \begin{itemize}
    \item[C1:] The support of $f_X$ is bounded, i.e. $\exists M > 0$ s.t. $Supp(f_X):= \{\x \in \R^p | f_X(\x) > 0\} \subset B_M(\0)$.
    \item[C2:] The density $f_X$ has bounded first two derivatives on the interior of $Supp(f_X)$.
    \item[C3:] There exist $l, H > 0$ s.t.
    \begin{align*}
        \sup_{\substack{\x \in Supp(f_X)\\0 < h_1 \leq h_2 \leq H}} \frac{f_X(\x)^2h_1^{p-1}h_2^{p-1}h_2^3}{\oint_{B_{h_1}(\x)}\oint_{B_{h_2}(\x)}f_X(\z)f_X(\w)d\z d\w} \leq l
    \end{align*}
    \item[C4:] For any $a > 0$ $\exists U > 0$ for which
    \begin{align*}
        \sup_{\substack{\x \in Supp(f_X),\\ 0 < u < U}} \frac{1}{\Fdx(u)^k} \int_1^\infty \Fdx( u\epsilon^{-1/a})^k d\epsilon = o(1),
    \end{align*}
    as $k \to \infty$.
\end{itemize}
Then
\begin{align*}
    I(\X_n, k(n)) \xrightarrow{P} I(f_X)
\end{align*}
provided the sequence $\{k(n)\}_{n \in \mathbb{N}}$ satisfies $k(n) \to \infty$ and $k(n)^{2-p/4}/n^{1-p/4} \to 0$ as $n\to\infty$.
\end{theorem}


\paragraph{Assumptions for Convergence}

Conditions C1 and C2 in the statement of Theorem~\ref{thm:conv} are very standard, however we have stated C3 and C4 in a form which is immediately useful in the proof of our result.  Nonetheless it is worthwhile briefly discussing these.

Note that for small $h_1, h_2$ we have $\oint_{B_{h_1}(\x)}\oint_{B_{h_2}(\x)}f_X(\z)f_X(\w)d\z d\w \approx a_0^2 f_X(\x)^2 h_1^{p-1}h_2^{p-1}$ and so C3 holds as long as there is an $H > 0$ independent of $\x$ for which this approximation holds reasonably well over $0 < h_1\leq h_2\leq H$. Note that the additional factor of $h_2^3$ adds an additional margin for this approximation.
On the other hand C4 essentially states that the distribution functions of the random variables $\{||X - \x|| \big| \x \in Supp(f_X)\}$ behave ``nicely'' for arguments close to zero. For example if there is a universal $U > 0$ for which all $\{\Fdx|\x \in Supp(f_X)\}$ are convex on the interval $[0, U]$ then we would have, for all $u \in (0, U)$, $\x \in Supp(f_X)$ and $\epsilon \geq 1$ that
\begin{align*}
    \Fdx(u\epsilon^{-1/a}) &\leq \Fdx(u)\epsilon^{-1/a}\\
    \Rightarrow \frac{1}{\Fdx(u)^k}\int_1^\infty \Fdx(u\epsilon^{-1/a})^k d\epsilon &\leq \int_1^\infty \epsilon^{-k/a}d\epsilon\\
    & = O(1/k),
\end{align*}
and so then this condition would hold strongly. Moreover if the density is smooth then we have for small $u$ that $\Fdx(u) \approx v_0 f_X(\x) u^p$, which is highly convex on the positive real numbers for even moderate $p$, and so the question would be whether or not there is an interval which is independent of $\x$ over which this approximation is close enough to maintain convexity.\\
\\
Perhaps more clearly, what is a far more common condition for nearest neighbours based estimation, in place of C3 and C4, is that the density is bounded away from zero on its support and that the support itself is relatively smooth. We show in the appendix that these conditions (along with C1 and C2) imply both C3 and C4. It is unclear to us, however, how much of a relaxation C3 and C4 is compared with the more common assumptions .\\
\\
The rest of this section constitutes a proof of Theorem~\ref{thm:conv}.

\subsection{Estimating $\frac{1}{f_X(\x)^2}\nabla f_X(\x)\nabla f_X(\x)'$ for Fixed $\x$}

As is common when studying the properties of nearest neighbours based estimators we consider first a fixed $\x$ in relation to a sample of size $n-1$ from $f_X$. For a given $\x \in \R^p$ define
    \begin{align*}
        \nnn{\x,n-1}{1}&= \argmin_{\z \in \X_{n-1}} ||\x - \z||,\\
        \nnn{\x,n-1}{k}&= \argmin_{\z \in \X_{n-1}\setminus \{\nnn{\x,n-1}{1}, ..., \nnn{\x,n-1}{k-1}\}} ||\x - \z||; \ k = 2, ..., n-1,
    \end{align*}
to be the ordering of the sample $\X_{n-1}$ based on their distances from $\x$, and let $\nnnd{\x,n-1}{k}:= ||\x - \nnn{\x,n-1}{k}||; k \in [n-1]$.
%
%
The usefulness of considering a fixed $\x$ is that the distribution of the random variable $g(\x, k, \X_{n-1}) := G(\x, \nnn{\x,n-1}{1}, ..., \nnn{\x,n-1}{k})$, where $G$ is some measureable function on $(\R^p)^{k+1}$, is equivalent to that of the conditional distribution of $g(X_i, k, \X_n\setminus \{X_i\})$, given $X_i = \x$, for any $i$.

Now, returning to fixed $\x$, let $\{j^*\}_{j\in [n-1]}$ be the rank of the distances in $\X_{n-1}$ from $\x$ so that $X_j = \nnn{\x,n-1}{j^*}$.
%
%
In the coming arguments we will use the fact that all dependence between the the elements $(\nnn{\x,n-1}{1}, ..., \nnn{\x,n-1}{n-1})$ is captured entirely through the relative ordering in their distances from $\x$. This means that $\nnn{\x, n-1}{k}$ is conditionally independent of $\X_{\x,n-1}^{(-k)}:= \{\nnn{\x,n-1}{1}, ..., \nnn{\x,n-1}{k-1}, \nnn{\x,n-1}{k+1}, \nnn{\x,n-1}{n-1}\}$, given $\nnnd{\x,n-1}{k}$; since the only information about $\nnn{\x,n-1}{k}$ which is carried by $\X_{\x,n-1}^{(-k)}$ is that $\nnnd{\x,n-1}{k-1} \leq \nnnd{\x,n-1}{k} \leq \nnnd{\x,n-1}{k+1}$ and this information is redundant, given $\nnnd{\x,n-1}{k}$. It also means that, conditional on its distance from $\x$, each $X_j$ is independent of the rank of its distance, since
\begin{align}
\nonumber
    X_j \big| (||X_j -\x||=h) &\overset{D}{=} X_j \big| (||X_j-\x||=h, \X_{-j})\\
\nonumber    &\overset{D}{=} X_j \big| (||X_j - \x||=h, \X_{-j}, j^*)\\
\nonumber    &\overset{D}{=} X_j \big| (||X_j - \x||=h, \X^{(-j^*)}, j^*)\\
    &\overset{D}{=} X_j \big| (||X_j - \x||=h, j^*).\label{eq:rank_independence}
\end{align}
In other words, the distributions of the nearest neighbours only differ by how the distributions of $\nnnd{\x,n-1}{j}; j\in[n-1]$, assign weights to the conditional distributions of the prototype $X \sim F_X$, given $||X - \x||$. Specifically, using $f_Z$ to denote the density of an arbitrary continuous random variable $Z$, we have
%
%
\begin{align*}
    f_{\nnn{\x,n-1}{j}}(\cdot) &= \int_0^\infty f_{\nnn{\x,n-1}{j}}(\cdot|\nnnd{\x,n-1}{j} = h)f_{\nnnd{\x,n-1}{j}}(h) dh\\
    &= \int_0^\infty f_X\left(\cdot\big | ||X -\x|| = h\right)f_{\nnnd{\x,n-1}{j}}(h) dh.
\end{align*}
With this understanding we are now in a position to establish that for $k$ small enough that $\nnnd{\x,n-1}{k}$ and $\nnnd{\x,n-1}{k+1}$ are small on average, we have
\begin{align}\label{eq:pointwise_approx}\nonumber
    E\Bigg[\frac{1}{\nnnd{\x,n-1}{k}^2\nnnd{\x,n-1}{k+1}^2}&(\nnn{\x,n-1}{k}-\nnn{\x,n-1}{k+1})(\nnn{\x,n-1}{k}-\nnn{\x,n-1}{k+1})'\Bigg]\\
    &\approx \frac{1}{p}E\left[\frac{1}{\nnnd{\x,n-1}{k}^2}+\frac{1}{\nnnd{\x,n-1}{k+1}^2}\right]\I - \frac{2}{p^2f_X(\x)^2} \nabla f_X(\x)\nabla f_X(\x)',
\end{align}
plus additional terms which depend on $f_X(\x)^{-1}\nabla^2 f_X(\x)$, where $\nabla^2 f_X(\x)$ is the Hessian of $f_X$ at $\x$. Towards elucidating the details, the importance of the conditional independence described previously is that we can address the left hand side in Eq.~(\ref{eq:pointwise_approx}) using the law of total expectation, where for $0<h_1 \leq h_2$, we have
\begin{align}
    \nonumber
    E&\Bigg[(\nnn{\x,n-1}{k}-\nnn{\x,n-1}{k+1})(\nnn{\x,n-1}{k}-\nnn{\x,n-1}{k+1})'\big|\nnnd{\x,n-1}{k}=h_1,\nnnd{\x,n-1}{k+1}=h_2\Bigg]\\
    \nonumber
    &= E\Bigg[\left(Z-W\right)\left(Z-W\right)'\bigg| ||Z-\x||=h_1,||W-\x||=h_2\Bigg],
\end{align}
where the expectation is over $Z, W$ being i.i.d. with distribution $F_X$. It is immediate then that this expectation may be written as
\begin{align*}
    \frac{1}{\oint_{B_{h_2}(\x)}\oint_{B_{h_1}(\x)}f_X(\z)f_X(\w)d\z d\w} \oint_{B_{h_2}(\x)}\oint_{B_{h_1}(\x)}(\z-\w)(\z-\w)'f_X(\z)f_X(\w)d\z d\w.
\end{align*}
Evaluating the numerator intergral, we find, using a standard Taylor expansion about~$\x$,
\begin{align*}
    \oint_{B_{h_2}(\x)}&\oint_{B_{h_1}(\x)}(\z-\w)(\z-\w)'f_X(\z)f_X(\w)d\z d\w = f_X(\x)^2\oint_{B_{h_2}(\0)}\oint_{B_{h_1}(\0)}(\z-\w)(\z-\w)'d\z d\w\\
    & + f_X(\x)\oint_{B_{h_2}(\0)}\oint_{B_{h_1}(\0)}(\z-\w)(\z-\w)'\nabla f_X(\x)'\z d\z d\w\\
    &+f_X(\x)\oint_{B_{h_2}(\0)}\oint_{B_{h_1}(\0)}(\z-\w)(\z-\w)'\nabla f_X(\x)'\w d\z d\w\\
    &+ \oint_{B_{h_2}(\0)}\oint_{B_{h_1}(\0)}(\z-\w)(\z-\w)'\nabla f_X(\x)'\w \nabla f_X(\x)'\z d\z d\w\\
    &+ \frac{1}{2}f_X(\x)\oint_{B_{h_2}(\0)}\oint_{B_{h_1}(\0)}(\z-\w)(\z-\w)'\left(\z'\nabla^2f_X(\x)\z + \w'\nabla^2f_X(\x)\w\right)d\z d\w\\
    &+ R_1,
\end{align*}
where since $f_X$ has a bounded second derivative we have $||R_1||_F < K_1 f_X(\x)(h_1+h_2)^2(h_1^2+h_2^2)h_1^{p-1}h_2^{p-1}$ for some $K_1>0$.
Now, using standard integration of polynomials on the unit sphere~\citep{polysphere} it is straightforward to verify that this is equal to
\begin{align*}
    a_0v_0&f_X(\x)^2h_1^{p-1}h_2^{p-1}(h_1^2+h_2^2)\I - 2v_0^2h_1^{p+1}h_2^{p+1}\nabla f_X(\x)\nabla f_X(\x)'\\
    & + \frac{f_X(\x)a_0^2}{p(p+2)}h_1^{p-1}h_2^{p-1}(h_1^4+h_2^4)\nabla^2 f_X(\x) + \frac{f_X(\x)a_0^2h_1^{p-1}h_2^{p-1}}{2p}\mathrm{tr}\left(\nabla^2 f_X(\x)\right)\left(\frac{h_1^4+h_2^4}{p+2}+\frac{h_1^2h_2^2}{p}\right)\I\\
    &+ R_1.
\end{align*}
Similarly, but far more straightforward, we have
\begin{align*}
    \oint_{B_{h_2}(\x)}\oint_{B_{h_1}(\x)}f_X(\z)f_X(\w)d\z d\w &= a_0^2f_X(\x)^2h_1^{p-1}h_2^{p-1} + R_2,
\end{align*}
where $|R_2| < K_2f_X(\x)(h_1^2+h_2^2)h_1^{p-1}h_2^{p-1}$ for some $K_2 > 0$.
Combining these we have
\begin{align*}
    E&\bigg[\frac{1}{\nnnd{\x,n-1}{k}^2 \nnnd{\x,n-1}{k+1}^2}(\nnn{\x,n-1}{k}-\nnn{\x,n-1}{k+1})(\nnn{\x,n-1}{k}-\nnn{\x,n-1}{k+1})'\big|\nnnd{\x,n-1}{k}=h_1,\nnnd{\x,n-1}{k+1}=h_2\bigg]\\
    &= \frac{v_0}{a_0}\left(\frac{1}{h_1^2} + \frac{1}{h_2^2}\right)\I - \frac{2v_0^2}{a_0^2f_X(\x)^2}\nabla f_X(\x)\nabla f_X(\x)' + R\\
    & \ + \frac{h_1^4+h_2^4}{p(p+2)f_X(\x)h_1^2h_2^2}\left(\nabla^2f_X(\x) + \frac{1}{2}\mathrm{tr}\left(\nabla^2 f_X(\x)\right)\I\right) + \frac{1}{2p^2f_X(\x)}\mathrm{tr}\left(\nabla^2 f_X(\x)\right)\I,
\end{align*}
where, 
\begin{align*}
    ||R||_F \leq K \frac{f_X(\x)h_2^4h_1^{p-1}h_2^{p-1}}{\oint_{B_{h_2}(\x)}\oint_{B_{h_1}(\x)}f_X(\z)f_X(\w)d\z d\w} \leq \frac{Klh_2}{f_X(\x)},
\end{align*}
for some $K > 0$ and where $l$ is given in Condition C3.

Integrating over the potential values of $h_1, h_2$ we thus have (noting that $v_0/a_0 = 1/p$),
\begin{align}\nonumber
    E&\bigg[\frac{1}{\nnnd{\x,n-1}{k}^2 \nnnd{\x,n-1}{k+1}^2}(\nnn{\x,n-1}{k}-\nnn{\x,n-1}{k+1})(\nnn{\x,n-1}{k}-\nnn{\x,n-1}{k+1})'\bigg]\\
    \nonumber
    &= \frac{1}{p}E\left[\frac{1}{\nnnd{\x, n-1}{k}^2} + \frac{1}{\nnnd{\x,n-1}{k+1}^2}\right]\I - \frac{2}{p^2f_X(\x)^2}\nabla f_X(\x)\nabla f_X(\x)'\\
    \nonumber
    & \ + \frac{1}{p(p+2)f_X(\x)}\left(\frac{\nnnd{\x, n-1}{k}^2}{\nnnd{\x, n-1}{k+1}^2}+\frac{\nnnd{\x, n-1}{k+1}^2}{\nnnd{\x, n-1}{k}^2}\right)\left(\nabla^2f_X(\x) + \frac{1}{2}\mathrm{tr}\left(\nabla^2 f_X(\x)\right)\I\right)\\
    & \ + \frac{1}{2p^2f_X(\x)}\mathrm{tr}\left(\nabla^2 f_X(\x)\right)\I + \tilde R, \label{eq:at_x}
\end{align}
where $||\tilde R||_F \leq \frac{Kl}{f_X(\x)}E[\nnnd{\x, n-1}{k+1}]$. Now, the reason for separating the terms involving $\frac{1}{f_X(\x)}\nabla^2 f_X(\x)$ is that the regularity conditions ensure that when taken over a random $X$, they contribute zero, i.e. $E[\frac{1}{f_X(X)}\nabla^2 f_X(X)] = \0$. It should be clear, however, that because of the dependence between the quantity $\frac{1}{f_X(X)}\nabla^2 f_X(X)$ and the nearest neighbour distances from $X$ this does not ensure that $E\left[\left(\frac{\nnnd{X, n-1}{k}^2}{\nnnd{X, n-1}{k+1}^2}+\frac{\nnnd{X, n-1}{k+1}^2}{\nnnd{X, n-1}{k}^2}\right)\frac{1}{f_X(X)}\nabla^2 f_X(X)\right]$, where the expectation is taken over $X, X_1, ..., X_{n-1} \sim_{i.i.d.} F_X$, is equal to $\0$. We address this in the following subsection.


%
%
%


\subsection{Averaging Over the Sample}

Using the expression in Eq.~(\ref{eq:at_x}), taken at each of the $X_i; i \in [n]$ and with neighbours taken from the rest of the sample, $\X_n$, we have
\begin{align}
    \nonumber
    E&\left[I(\X_n, k)\right]= I(f_X) + R^*\\
    &- \frac{p}{2(p+2)}E\left[\left(\frac{\nnnd{1,n}{k}^2}{\nnnd{1,n}{k+1}^2} + \frac{\nnnd{1,n}{k+1}^2}{\nnnd{1,n}{k}^2}\right)\frac{1}{f_X(X_1)}\left(\nabla^2 f_X(X_1) + \frac{1}{2}\mathrm{tr}\left(\nabla^2 f_X(X_1)\right)\right)\right], \label{eq:unbiased}
\end{align}
where $||R^*||_F \leq K^*E[\frac{\nnnd{1, n}{k+1}}{f_X(X_1)}]$ for some $K^* > 0$, and we have used the index ``1'' (i.e. the first sample point) arbitrarily, since each term in the sum defining $I(\X_n, k)$ has the same distribution. Now, it is well known that if $Supp(f_X)$ is bounded then $\sup_{\x \in Supp(f_X)}E\left[\nnnd{\x,n-1}{k+1}\right]$ vanishes as $n \to \infty$ and hence $E[\frac{\nnnd{1,n}{k+1}}{f_X(X_1)}] = \int_{Supp(f_X)}E[\nnnd{\x,n-1}{k+1}]d\x$ tends to zero as $n \to \infty$. In addition the following lemma, the proof of which may be found in the appendix, can be used to ensure that $E\left[\frac{\nnnd{\x,n-1}{k+1}}{\nnnd{\x,n-1}{k}} + \frac{\nnnd{\x,n-1}{k}}{\nnnd{\x,n-1}{k+1}}\right]$ converges to 2, uniformly in $\x$, as $n \to \infty$.

\begin{lemma}\label{lem:nnratio}
    Let $X_1, X_2, ...$ be i.i.d. random variables with density $f_X$ satisfying Conditions C1--C4. Let $0 < a \leq b$ and let $\{k(n)\}_{n\in \mathbb{N}}$ satisfy $\lim_{n \to \infty} k(n) = \infty$ and $\lim_{n\to\infty} k(n)/n = 0$. Then
    \begin{align*}
        \lim_{n \to \infty} \sup_{\x \in Supp(X)} \Bigg(E\left[\frac{\nnnd{\x,n}{k(n)+1}^b}{\nnnd{\x,n}{k(n)}^a}\right] - E\left[\nnnd{\x,n}{k(n)+1}^{b-a}\right]\Bigg) &= 0.
    \end{align*}
\end{lemma}
Notice that trivially we have
\begin{align*}
    E\left[\frac{\nnnd{\x,n}{k(n)+1}^b}{\nnnd{\x,n}{k(n)}^a}\right] - E\left[\nnnd{\x,n}{k(n)+1}^{b-a}\right] \geq 0,
\end{align*}
and so applying the lemma for $b=a=2$ we see that $E\left[\frac{\nnnd{\x,n-1}{k+1}^2}{\nnnd{\x,n-1}{k}^2}\right]$ converges uniformly to 1 as $n \to \infty$. To see that $E\left[\frac{\nnnd{\x,n-1}{k}^2}{\nnnd{\x,n-1}{k+1}^2}\right]$ also converges to one notice simply that $1 \geq E\left[\frac{\nnnd{\x,n-1}{k}^2}{\nnnd{\x,n-1}{k+1}^2}\right] = E\left[\left(\frac{\nnnd{\x,n-1}{k+1}^2}{\nnnd{\x,n-1}{k}^2}\right)^{-1}\right] \geq E\left[\frac{\nnnd{\x,n-1}{k+1}^2}{\nnnd{\x,n-1}{k}^2}\right]^{-1}$ which converges uniformly to 1. We therefore also have
\begin{align*}
    \Bigg \|E\bigg[&\left(\frac{\nnnd{1,n}{k}^2}{\nnnd{1,n}{k+1}^2} + \frac{\nnnd{1,n}{k+1}^2}{\nnnd{1,n}{k}^2}\right)\frac{1}{f_X(X_1)}\left(\nabla^2 f_X(X_1) + \frac{1}{2}\mathrm{tr}\left(\nabla^2 f_X(X_1)\right)\right)\bigg] \Bigg \|_F\\
    =& \Bigg \|\int\limits_{Supp(f_X)} \left(\nabla^2 f_X(\x) + \frac{1}{2}\mathrm{tr}\left(\nabla^2 f_X(\x)\right)\right)E\left[\frac{\nnnd{\x,n-1}{k+1}}{\nnnd{\x,n-1}{k}} + \frac{\nnnd{\x,n-1}{k}}{\nnnd{\x,n-1}{k+1}}\right] d\x \Bigg \|_F\\
    \leq & 2 \Bigg \|\int\limits_{Supp(f_X)} \left(\nabla^2 f_X(\x) + \frac{1}{2}\mathrm{tr}\left(\nabla^2 f_X(\x)\right)\right) d\x \Bigg \|_F\\
    & + \int\limits_{Supp(f_X)}\left\|\nabla^2 f_X(\x) + \frac{1}{2}\mathrm{tr}\left(\nabla^2 f_X(\x)\right)\right\|_F \Bigg| E\left[\frac{\nnnd{\x,n-1}{k+1}}{\nnnd{\x,n-1}{k}} + \frac{\nnnd{\x,n-1}{k}}{\nnnd{\x,n-1}{k+1}}\right] - 2\Bigg| d\x
\end{align*}
where the first term is exactly zero and the second term converges to zero as $n \to \infty$ since $\nabla^2 f_X(\x)$ is bounded and $E\left[\frac{\nnnd{\x,n-1}{k+1}}{\nnnd{\x,n-1}{k}} + \frac{\nnnd{\x,n-1}{k}}{\nnnd{\x,n-1}{k+1}}\right]$ converges uniformly to 2, as described. We therefore have $\lim_{n \to \infty} E[I(\X_n, k(n))] = I(f_X)$.

Now, to bound the variance of $I(\X_n, k)$ we use Lemma 4.6 of~\cite{Dafydd}, which we apply to each of the elements of $I(\X_n, k)$, giving
\begin{align*}
    Var\left(I(\X_n,k)_{i,j}\right) \leq \frac{2(n+1)(3+8k^2pv_0)}{n^2}E\left[\Phi_{X,n-1,(i,j)}^2\right],
\end{align*}
where $\Phi_{X,n-1,(i,i)} = \frac{p}{2}\left(\frac{1}{\nnnd{X,n-1}{k}^{2}} + \frac{1}{\nnnd{X,n-1}{k+1}^{2}} - \frac{p(\nnn{X,n-1}{k}-\nnn{X,n-1}{k+1})_i^2}{\nnnd{X,n-1}{k}^2\nnnd{X,n-1}{k+1}^2}\right)$ and for $i\not=j$, $\Phi_{X,n-1,(i,j)} = \frac{p^2(\nnn{X,n-1}{k}-\nnn{X,n-1}{k+1})_i(\nnn{X,n-1}{k}-\nnn{X,n-1}{k+1})_j}{2\nnnd{X,n-1}{k}^2\nnnd{X,n-1}{k+1}^2}$. Now, clearly we have, for any $i,j$, that
\begin{align*}
    |(\nnn{X,n-1}{k}-\nnn{X,n-1}{k+1})_i(\nnn{X,n-1}{k}-\nnn{X,n-1}{k+1})_j| &\leq ||\nnn{X,n-1}{k}-\nnn{X,n-1}{k+1}||^2 \leq \left(\nnnd{X,n-1}{k} + \nnnd{X,n-1}{k+1}\right)^2\\
    \Rightarrow \frac{|(\nnn{X,n-1}{k}-\nnn{X,n-1}{k+1})_i(\nnn{X,n-1}{k}-\nnn{X,n-1}{k+1})_j|}{\nnnd{X,n-1}{k}^2\nnnd{X,n-1}{k+1}^2} &\leq 2(\nnnd{X,n-1}{k}^{-2}+\nnnd{X,n-1}{k+1}^{-2}) \leq 4\nnnd{X,n-1}{k}^{-2},
\end{align*}
and hence $\Phi_{X,n-1,(i,j)}^2\leq \frac{p^2(2+4p)^2}{4}\nnnd{X,n-1}{k}^{-4}$. It is therefore sufficient that $E[\nnnd{X,n-1}{k}^{-4}] = o(\frac{n}{k^2})$ in order to ensure $Var\left(I(\X_n, k)_{i,j}\right)$ vanishes asymptotically. To see this, note that using the properties of order statistics we have
\begin{align*}
    E\left[\nnnd{\x,n-1}{k}^{-4}\right] &= \frac{n!}{(k-1)!(n-k)!}\int_0^\infty u^{-4}\Fdx(u)^{k-1}(1-\Fdx(u))^{n-k}d\Fdx(u).
\end{align*}
Now, both a bounded derivative and bounded support of $f_X$ ensures $\fupper := \sup_{\x}f_X(\x) < \infty$, and hence $\Fdx(u) \leq v_0\fupper u^p$, giving
\begin{align*}
    E\left[\nnnd{\x,n-1}{k}^{-4}\right] &\leq \frac{n!}{(k-1)!(n-k)!}(v_0\fupper)^{4/p}\int_0^\infty \Fdx(u)^{k-1-4/p}(1-\Fdx(u))^{n-k}d\Fdx(u)\\
    &= (v_0\fupper)^{4/p}\frac{n!}{(k-1)!(n-k)!} \frac{\Gamma(k-4/p)(n-k)!}{\Gamma(n-4/p)}\\
    &\leq (v_0\fupper)^{4/p}\left(\frac{n+1}{k}\right)^{4/p},
\end{align*}
where in the final step we have used Gautschi's inequality~\citep{Gautschi}. Provided $\frac{k(n)^{2-4/p}}{n^{1-4/p}} \to 0$ we have $E[\nnnd{X,n-1}{k(n)}^{-4}] = o(n/k(n)^2)$, and hence $Var\left(I(\X_n, k(n))_{i,j}\right) \to 0$ as $n \to \infty$ as required.

We have thus shown that $I(\X_n, k(n))$ is asymptotically unbiased for $I(f_X)$, and its variance vanishes as $n \to \infty$, and so $I(\X_n, k(n)) \xrightarrow[]{P}I(f_X)$, and we are done.

\section{Practicalities}\label{sec:practicalities}

As mentioned in Section~\ref{sec:intro}, when performing dimension reduction we perform a few modifications to the basic approach of projecting onto the eigenvectors of $I(\calX, k)\Sigma_\calX$. The first two are simply to improve the quality of $I(\calX, k)$ as an estimate of $I(f_X)$, whereas the third is a straightforward orthogonalisation of the eigenvectors which we have found to lead to improved results when performing downstream tasks such as clustering and outlier detection on the projected data.

\subsection{Improving Estimation}

\subsubsection{Variance Reduction}

To obtain a lower variance estimator of the DIM we simply average over multiple estimates, obtained for different values of $k$. Specifically, for $0 < k_0 \leq k_1 < n$ we define $I(\calX, k_0, k_1) := \frac{1}{k_1-k_0+1}\sum_{k=k_0}^{k_1}I(\calX, k)$. It should be clear that, provided $k_0(n)$ and $k_1(n)$ both satisfy the conditions of Theorem~\ref{thm:conv}, we have $I(\X_n, k_0(n), k_1(n)) \xrightarrow[]{P}I(f_X)$. That is, averaging over an appropriate set of values for $k$ does not affect consistency of the estimator.

\subsubsection{Bias Reduction}

To reduce the bias of $I(\calX, k_0, k_1)$ as an estimate for $I(f_X)$ we rely on the theoretical properties of $I(f_X)$ given by~\cite{DIM}. 
Since $I(\calX, k_0, k_1)$ is symmetric, we can express it via its spectral decomposition as $\U \D \U'$, where the columns of $\U$ are its eigenvectors and $\D$ is the diagonal matrix containing its eigenvalues. We then define $I(\calX, k_0, k_1)^*:= \U\D^*\U'$ where $\D^*$ is diagonal with
\begin{align*}
    \D_{ii}^* = \max\{\D_{ii}, (\u_i'\Sigma_{\calX}\u_i + \epsilon)^{-1}\},
\end{align*}
where $\u_i; i \in [p]$, the $i$-th column of $\U$ and $\epsilon$ is a small additive factor to avoid numerical issues\footnote{we set $\epsilon = 10^{-7}$ in our experiments but have found the method to be very insensitive to this setting.}. This adjustment is based on the following lemma, which describes a lower bound on the spectrum of the population level DIM.

\begin{lemma}
Let $I(f_X) = \V\Delta\V'$ be the standard spectral decomposition of $I(f_X)$. Then $\Delta_{ii} \geq \frac{1}{Var(\v_i'X)}$, where $\v_i; i \in [p]$, is the $i$-th column of $\V$.
\end{lemma}

\begin{proof}
    Follows from Eq.~(4) and Eq.~(6) from \cite{DIM}.
\end{proof}

Notice also that if the eigenvalues of $I(f_X)$ are distinct then the spectral decomposition of $I(\X_n, k_0(n), k_1(n))$ converges to that of $I(f_X)$, and so for any fixed $\epsilon > 0$ the probability that a bias adjustment actually modifies $I(\X_n, k_0(n), k_1(n))$ converges to zero. As a consequence we clearly have $I(\X_n, k_0(n), k_1(n))^* \xrightarrow[]{P} I(f_X)$.



\subsection{Orthogonalisation}

The final modification is a simple orthogonalisation of the eigenvectors of $I(\calX, k_0, k_1)^*\Sigma_\calX$. The most notable implications of this are from the point of view of overall variation in the projected data, where if $I(\calX, k_0, k_1)^*\Sigma_\calX = \W\Gamma\W^{-1}$ is the eigen-decomposition of $I(\calX, k_0, k_1)^*\Sigma_\calX$ it is possible that trace$(\Sigma_{\W'\calX}) = $ trace$(\W'\Sigma_\calX \W)$ is substantially less than trace$(\Sigma_\calX)$. We therefore use as our final projection basis the columns of $\W^* = \W_\u\W_\v'$, where $\W_\u$ and $\W_\v$ contain in their columns the left and right singular vectors of $\W$, respectively.

\section{Experiments}\label{sec:experiments}

In this section we document the practical utility of the proposed approach, by exploring its use as a dimension reduction tool to aid downstream tasks such as cluster analysis and outlier detection. For comparisons we use a number of other general purpose linear dimension reduction techniques, including (i) standard PCA; (ii) robust PCA (RPCA) using the method of~\cite{robpca}; (iii) Locality Preserving Projections~\citep[LPP]{lpp}; and (iv) the kernel based DIM method (kDIM) of~\cite{DIM}. We also investigated methods designed specifically for clustering~\citep{hofmeyrPPCI} and outlier detection~\citep{dobin}, however we found the general purpose methods outperformed these even on problems for which they were designed, and so omit their results.

For RPCA we used the implementation in the {\tt R} package {\tt rrcov}~\citep{rrcov}, and when $p$ exceeds 50 we used the projection pursuit approach due to the high computational cost of computing the minimum covariance discriminant in high dimensions. For LPP we explored the implementation in the {\tt R} package {\tt Rdimtools}, but found far better performance using our own implementation using a nearest neighbour graph with $\lfloor \log(n)\rfloor$ neighbours\footnote{We found the implementation to be fairly insensitive to the number of neighbours, but performance was slightly worse for larger settings}. For kDIM we converted MATLAB code obtained from the authors for use within {\tt R}. For the proposed approach we used $k_1 = \lceil 2\log(n)\rceil$ and $k_0 = \lfloor k_1/2\rfloor$. 
Both PCA and RPCA produce orthogonal projection bases, whereas LPP and kDIM produce non-orthogonal bases. We explored both the raw and orthogonalised variants and found that for LPP the raw basis was superior, while for the kDIM method orthogonalisation improved performance. We report only the results from the superior variant of each.

Rather than attempting to separately determine an appropriate number of dimensions to extract with the different methods, we simply reduced dimensionality from $p$ to $\lceil p^{4/5}\rceil$ in all cases.



\paragraph{Preprocessing}

Distance based methods, such as those based on nearest neighbours, are highly sensitive to the scaling of individual variables relative to one another. It is therefore crucial that this is accounted for before the application of any such methods. We used a standard scaling approach of dividing each variable by the sample estimate of its standard deviation. Note that a complete whitening of the data, which results in unit variance when measured in any direction and not only along the cardinal basis directions, is a reasonable alternative. We prefer the simpler standardisation since when the data lie close to a low dimensional subspace total whitening can have the effect of inflating the noise in directions orthogonal to this subspace. Moreover, as it is discussed by~\cite{DIM0, DIM}, at a population level the eigenvectors of $I(f_X)\Sigma$ are proportional to $\Sigma^{-1/2}\mathbf{W}$, where $\mathbf{W}$ are the eigenvectors of the DIM of the whitened $\Sigma^{-1/2}X$. Again at a population level, first whitening and then projecting onto the eigenvectors of the DIM of the whitened $\Sigma^{-1/2}X$ is equivalent (up to scaling of the resulting projected variables) to projecting directly onto the eigenvectors of $I(f_X)\Sigma$.

For consistency across different methods, we standardised the data before applying any of the dimension reduction techniques.

\subsection{Estimation Accuracy}

Before exploring the utility of the proposed approach for dimension reduction, we first briefly explore the estimation accuracy of $I(\calX, k_0, k_1)^*$ in the simple case where $X$ has a Gaussian distribution, and hence where $I(f_X)$ is known to be equal to $\Sigma^{-1}$. Since $I(\calX, k_0, k_1)^*$ and $I(f_X)$ are location invariant we fix $\boldsymbol{\mu} = \0$ throughout for simplicity. To produce a variety of covariance matrices, in each experimental set-up we generate covariance matrices randomly by sampling the set of eigenvectors uniformly from the Stiefel manifold, and the set of eigenvalues independently from an Exponential$(1)$ distribution. For simple comparisons we use $\Sigma_{\calX}^{-1}$ and $\tilde\Sigma_{\calX}^{-1}$, where $\Sigma_{\calX}$ is the maximum likelihood estimator for $\Sigma$ and $\tilde\Sigma_{\calX}$ is the regularised estimator described by~\cite{schafer2005shrinkage}. For the latter we used the implementation in the {\tt R} package {\tt corpcor}~\citep{corpcor}, and the automatic selection of regularisation strength.

We vary $n$ in $\{1000, 2000\}$ and $p$ in $\{10, 20, 50, 100, 200, 500\}$. For each pair of $p$ and $n$ we generate 50 covariance matrices as described above, and a single sample of size $n$ for each, from which to perform estimation. To assess the accuracy of estimation we report the average (and standard deviation of) values of $\log\left(||\hat \tau - \tau||_F^2\right) - \log\left(||\tau||_F^2\right)$, where $\tau = \Sigma^{-1}$ and $\hat \tau$ is one of $I(\calX, k_0, k_1)^*; \Sigma_\calX^{-1}$ and $\tilde \Sigma_\calX^{-1}$. Results from these experiments are shown in Figure~\ref{fig:accuracy}, where the average accuracy from estimation using the maximum likelihood estimator is shown with \textcolor{red}{--$\scriptstyle\triangle$--}; using the regularised estimator with \textcolor{green}{--$\scriptstyle\square$--}; and using $I(\calX, k_0, k_1)^*$ with --$\circ$--. $I(\calX, k_0, k_1)^*$ shows quite competitive accuracy to using the maximum likelihood estimator for all values of $p$, and even outperforms it for larger values of $p$. The regularised variant is primarily used for large $p$ relative to $n$ scenarios, and is only included for interested readers.

\begin{figure}
    \centering
    \subfigure[$n=1000$]{\includegraphics[width=0.48\linewidth]{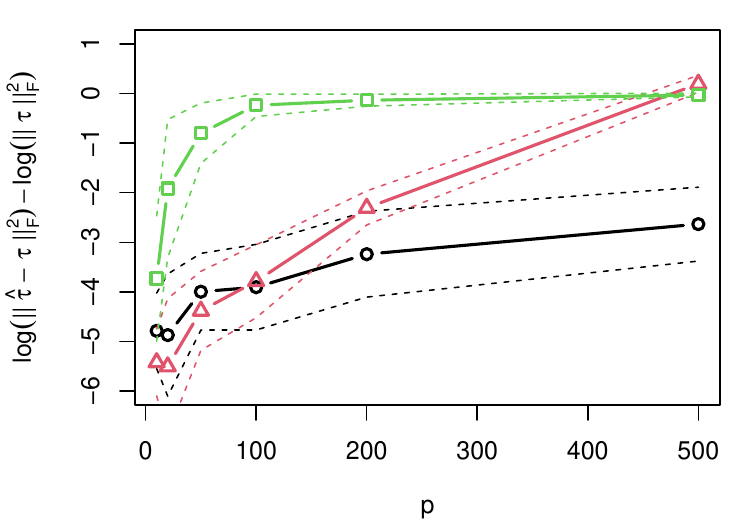}}
    \subfigure[$n=2000$]{\includegraphics[width=0.48\linewidth]{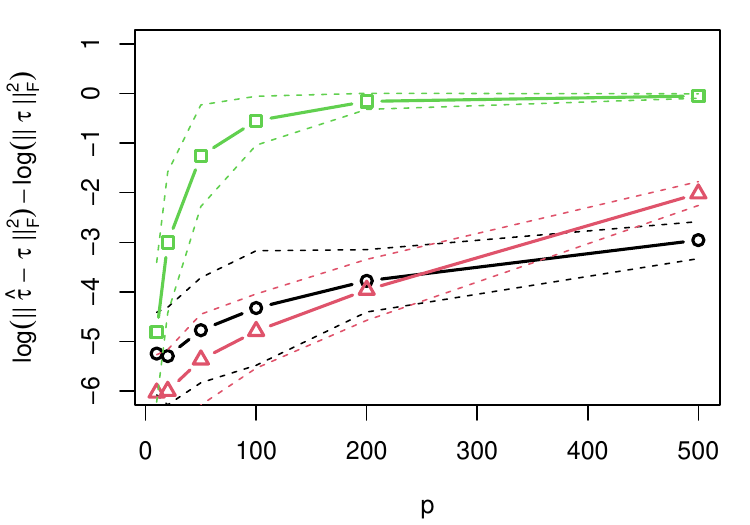}}
    \caption{Average estimation error for $I(f_X)$ when $X\sim N(\0, \tau^{-1})$ using (i) $\Sigma_\calX^{-1}$ (\textcolor{red}{--$\scriptstyle\triangle$--}); (ii) $\tilde \Sigma_\calX^{-1}$ (\textcolor{green}{--$\scriptstyle\square$--}); and (iii) $I(\calX, k_0, k_1)^*$ (--$\circ$--). The dashed lines indicate one standard error deviations from the mean.}
    \label{fig:accuracy}
\end{figure}

\subsection{Clustering}

To assess the usefulness of the proposed approach for aiding cluster analysis we applied it, plus the other general purpose dimension reduction techniques, to the 45 data sets used by~\cite{hofmeyrCNS}\footnote{two of the data sets have multiple potential groupings, and we simply treated each of these as separate clustering problems, leading to a total of 48}. These are data for which ``ground truth'' classifications of the points into groups are available, and so clustering performance can be quantified by how well the clusters align with these true groups/classes.

To perform clustering we used the ubiquitous KMeans; the hierarchical variant of DBSCAN~\citep[HDBSCAN]{campello2013density}; Spectral Clustering~\citep[SC]{ng2001spectral}; and the recently proposed Torque Clustering~\citep[TORC]{TORC}. For KMeans we used the popular KMeans++\citep{arthur2007k} initialisation and selected the number of clusters using the silhouette score~\citep{kaufman2009finding}. For HDBSCAN we used the implementation in the {\tt R} package {\tt dbscan} and simply set the threshold number of neighbours to characterise a point as a high density point to 5. For SC we used the implementation in the {\tt R} package {\tt kknn}~\citep{kknn} as it is both computationally efficient and can automatically determine an appropriate number of clusters to extract from a set of data. Finally, TORC is a fully automatic clustering method which does not require any specification of the number of clusters, nor any hyperparameters. To quantify the quality of a clustering solution we use the Adjusted Rand Index~\citep[ARI]{hubert1985comparing}. We also considered the Adjusted Mutual Information, but the results are very similar and so for brevity we report only on the ARI.

\begin{figure}
    \centering
    \subfigure[KMeans]{\includegraphics[width=0.99\linewidth]{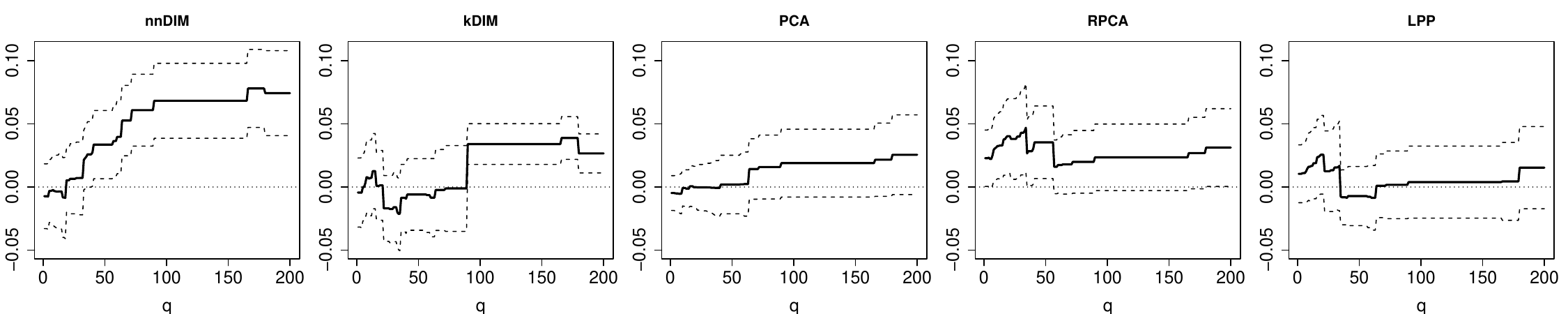}}
    \subfigure[HDBSCAN]{\includegraphics[width=0.99\linewidth]{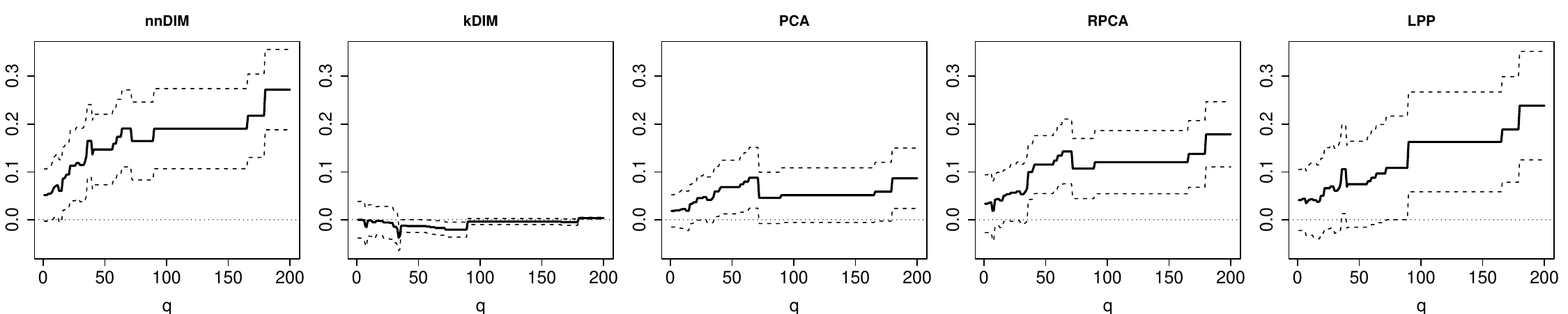}}
    \subfigure[Spectral Clustering]{\includegraphics[width=0.99\linewidth]{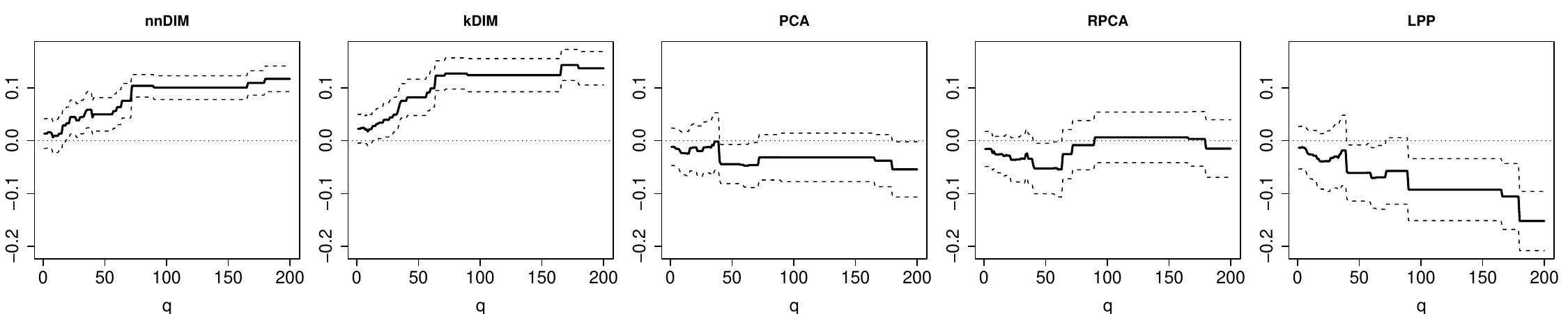}}
    \subfigure[TORC]{\includegraphics[width=0.99\linewidth]{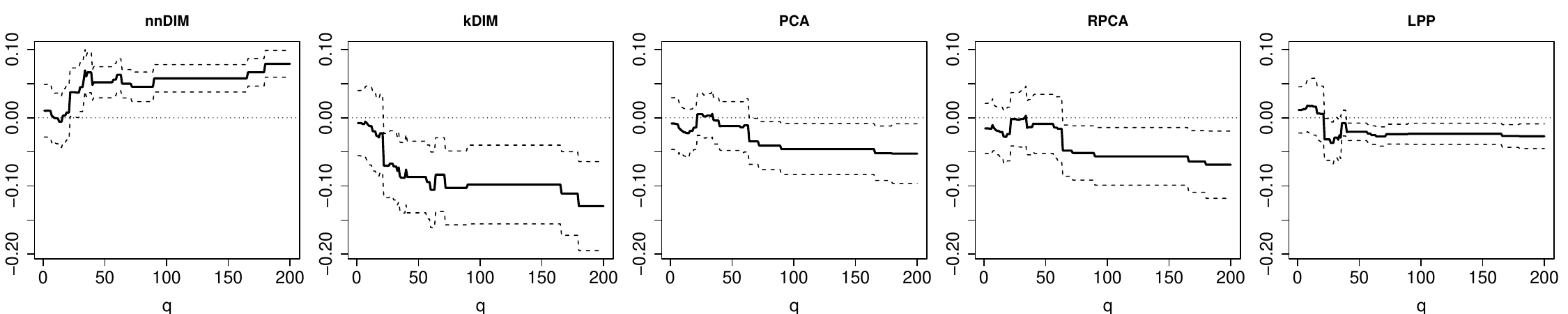}}
    \caption{Each plot shows, for each $q \in \{1, 2, ..., 200\}$, the average increase in ARI due to dimension reduction, using only the subset of data sets with dimension at least $q$. Dashed lines correspond with two standard error deviations. Note that there are six data sets with dimension at least 200, i.e. the rightmost value in each plot is an average of six ARI differences while the leftmost is an average of all 48.}
    \label{fig:ARIq}
\end{figure}

The results are summarised in Figure~\ref{fig:ARIq}, where we have used nnDIM (nearest neighbours DIM) to refer to the proposed approach. The plots in the figure show the average differences between the ARI achieved by clustering the data after dimension reduction and that achieved by clustering the raw data, when taken over those data sets above the given dimension. That is, for a given value of $q$, we filtered out data sets with dimension below $q$ and computed the average ARI difference over the remaining data sets. We then plotted these averages as a function of $q$. In particular for $q = 1$ (at the leftmost point in each plot) the average includes all data sets, whereas for $q = 200$ (at the rightmost) only the six data sets of highest dimension, i.e. those above 200 dimensions, are included in the average. 

It is noteworthy that, when filtering out the data sets of lower dimension, the proposed approach is the only one of the five which leads to an improvement in clustering quality for all four clustering methods.

\subsection{Outlier Detection}

To assess the usefulness of the proposed approach for aiding outlier detection, we applied it, plus the other dimension reduction techniques, to a large collection of data sets derived from the same set of 45 classification data sets used previously (48 after including the multiple label sets). We used the common approach of splitting the classes in a data set into ``inlier classes'' and ``outlier classes'', and then creating a data set with outliers by combining all observations in the inlier classes with a small sample of the observations from the outlier classes. For simplicity we used a fixed ratio of ``outliers'' to ``inliers'' of 1:20.

We considered three approaches to splitting the classes into inlier classes and outlier classes: (i) each class is treated, in turn, as the inlier class, and all other classes are outlier classes; (ii) all pairs of classes are treated, in turn, as inlier classes, and all remaining classes are outlier classes; and (iii) each class is treated, in turn, as the only outlier class, and all other classes are inlier classes. For (i) we used all 48 classification data sets (after including the multiple label sets), whereas due to the large number of pairs of classes in some of the data sets we excluded those with more than ten classes for (ii). We also excluded data sets with only two classes from (ii). For (iii), we included only data sets with up to five classes. The reason for this is that if the number of classes is large, then a sample from the single outlier class comprising close to 5\% of the total number of observations will better resemble a cluster than a small collection of outliers. Finally, we excluded from the resulting data sets any instances where the total number of inliers was less than 30. This resulted in 274 data sets of type (i); 545 of type (ii); and 69 of type (iii).\\
\\
For outlier ``detection'' we used the rankings based on isolation forest~\citep[IF]{liu2008isolation}; simplified Local Outlier Factors~\citep[sLOF]{schubert2014local}; and the simple nearest neighbour distance method of~\cite{angiulli2002fast} which uses the sum of the $k$ nearest neighbour distances from a point to determine its ``outlierness'' (kNNW); where we simply set $k = 10$ throughout, deliberately choosing a value different from the value used in evaluating $I(\calX, k_0, k_1)^*$ to avoid the possibility that the same setting may artificially inflate the performance of the proposed approach. In addition, to reduce dependence on the hyperparameter determining locality in sLOF, which uses nearest neighbour distances, we also average the sLOF scores from settings of $k$ in $\left\{1, 2, ..., 10\right\}$. To evaluate the performance of a ranking of points based on their ``outlierness'' we used the area under the precision-recall curve (AUPRC).

\begin{figure}[t]
    \centering
    \subfigure[kNN Weight]{\includegraphics[width=0.99\linewidth]{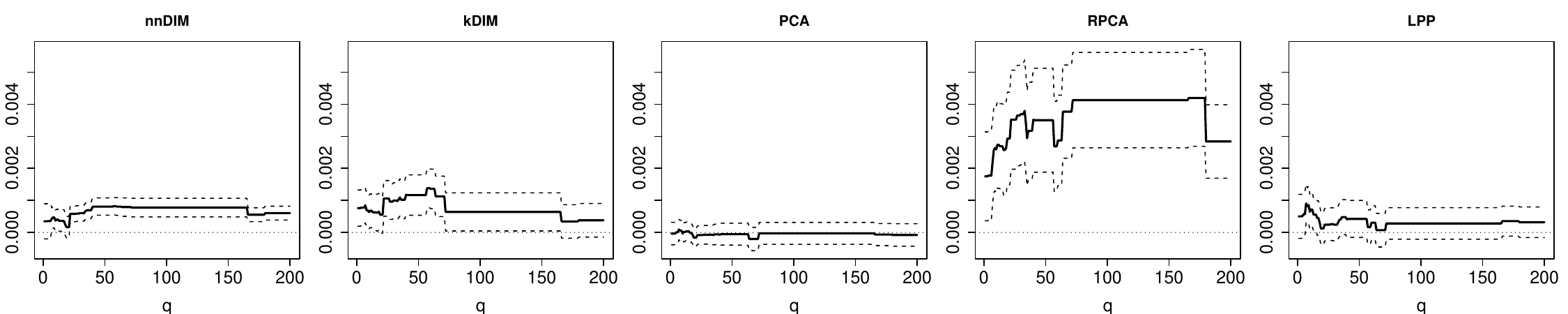}}
    \subfigure[Simplified LOF]{\includegraphics[width=0.99\linewidth]{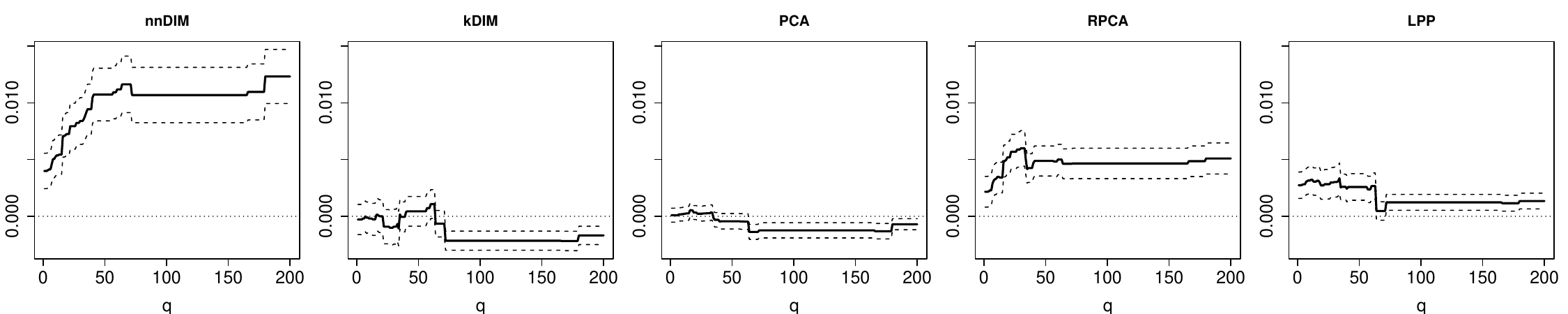}}
    \subfigure[Isolation Forest]{\includegraphics[width=0.99\linewidth]{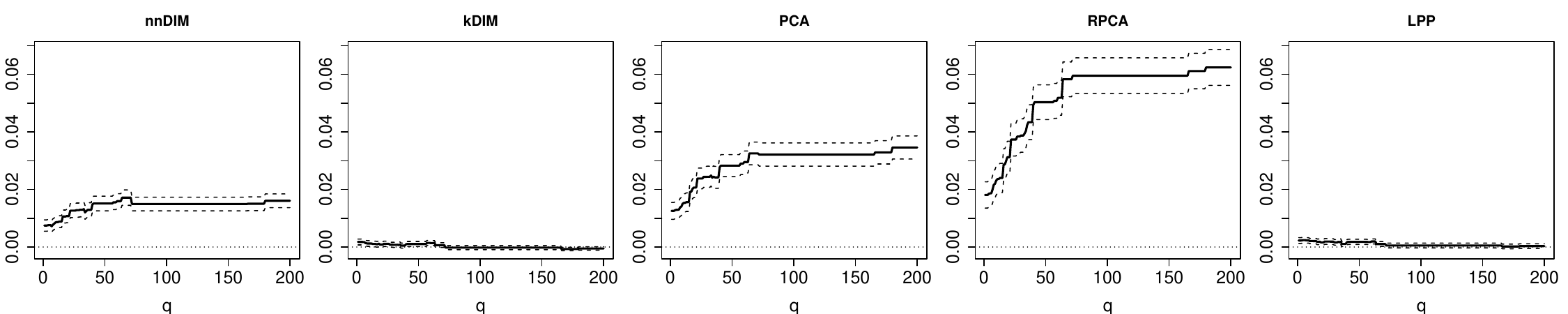}}
    \caption{Single inlier class data sets (type (i)): Average increases in AUPRC due to dimension reduction, after filtering out data sets of dimension below $q$, for $q \in \{1, 2, ..., 200\}$. Dashed lines correspond with two standard error deviations.}
    \label{fig:outlier_typei}
\end{figure}

\begin{figure}[t]
    \centering
    \subfigure[kNN Weight]{\includegraphics[width=0.99\linewidth]{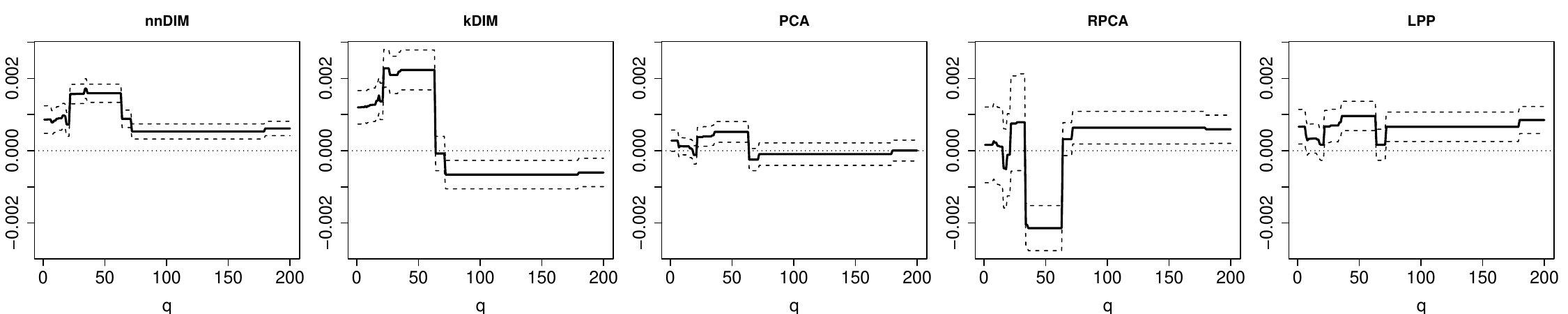}}
    \subfigure[Simplified LOF]{\includegraphics[width=0.99\linewidth]{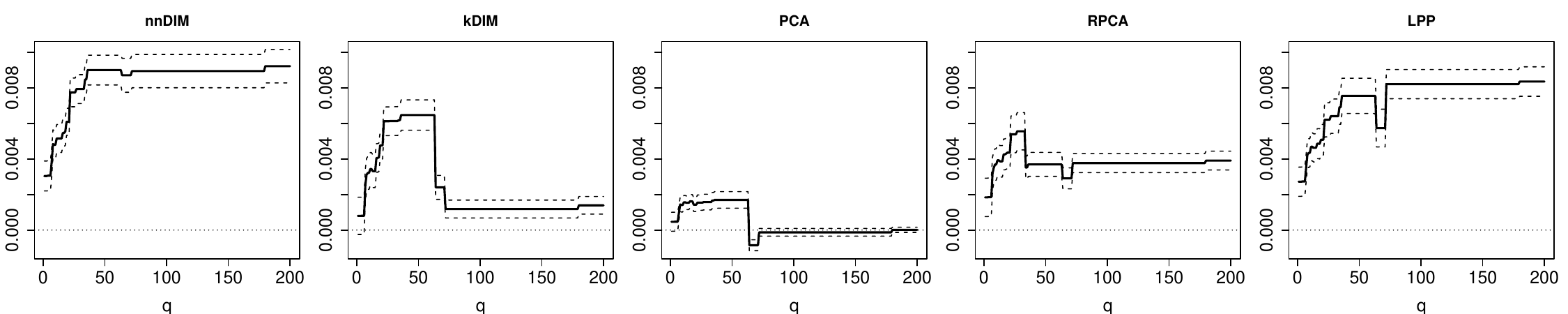}}
    \subfigure[Isolation Forest]{\includegraphics[width=0.99\linewidth]{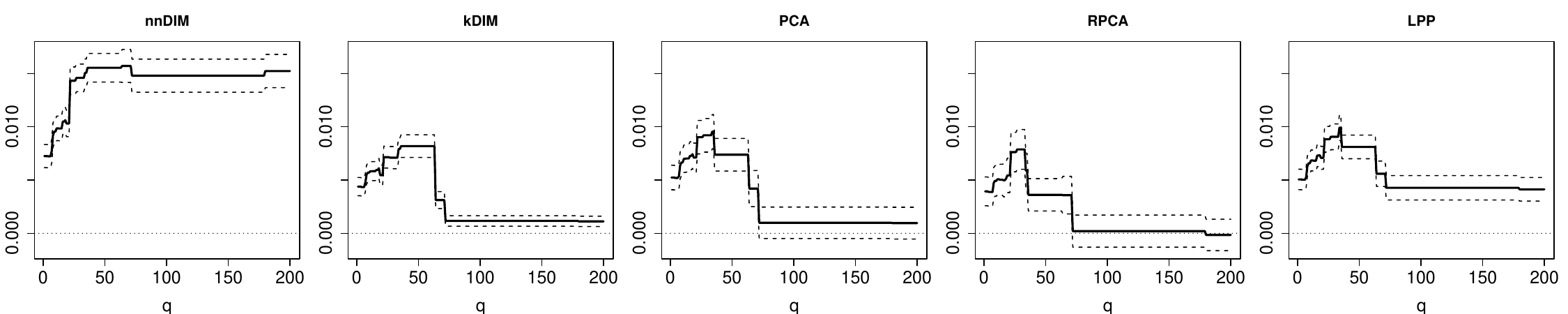}}
    \caption{Two inlier class data sets (type (ii)): Average increases in AUPRC due to dimension reduction, after filtering out data sets of dimension below $q$, for $q \in \{1, 2, ..., 200\}$. Dashed lines correspond with two standard error deviations.}
    \label{fig:outlier_typeii}
\end{figure}

The results are summarised in Figures~\ref{fig:outlier_typei}~-~\ref{fig:outlier_typeiii}, where the results are analogous to those in Figure~\ref{fig:ARIq} in that points in the plots show the average differences in performance, between results from the reduced data and those from the raw data, after filtering out data sets with dimension less than $q$, for $q$ up to 200. On data sets with a single inlier class, Figure~\ref{fig:outlier_typei}, RPCA arguably shows the best overall performance but with the proposed approach the best pairing with sLOF. With two inlier classes, Figure~\ref{fig:outlier_typeii}, the proposed approach is the best pairing for both sLOF and Isolation Forest, while it is unclear what the best pairing is with kNNW. Finally, when there is only one outlier class, Figure~\ref{fig:outlier_typeiii}, none of the results is as conclusive as in the previous cases, with the bounds of the naive confidence intervals formed by the two standard error bands typically transcending or lying close to zero. None of the methods improves performance of kNNW when compared with using the raw data, while the proposed approach shows some improvement over the raw data when combined with sLOF, but is one of only two (along with kDIM) which do not appear to offer an improvement to Isolation Forests.

\begin{figure}[t]
    \centering
    \subfigure[kNN Weight]{\includegraphics[width=0.99\linewidth]{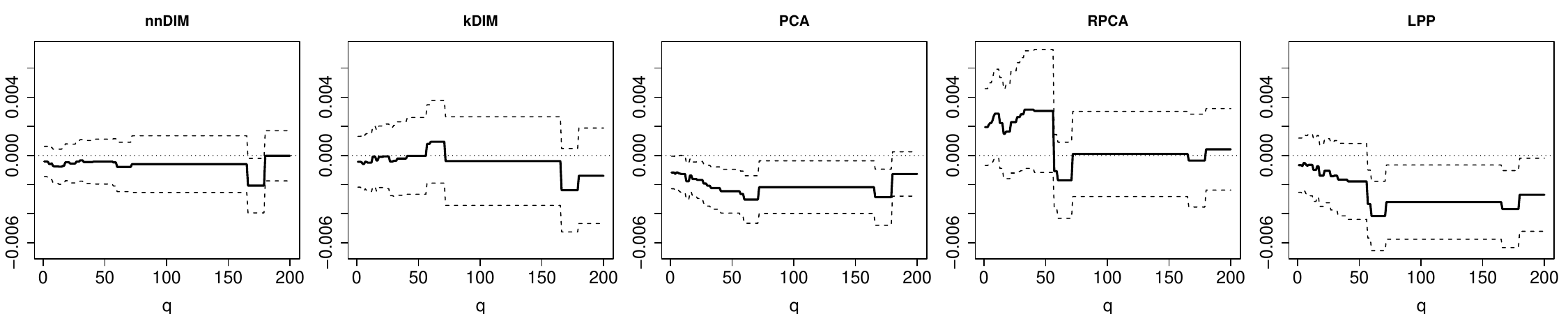}}
    \subfigure[Simplified LOF]{\includegraphics[width=0.99\linewidth]{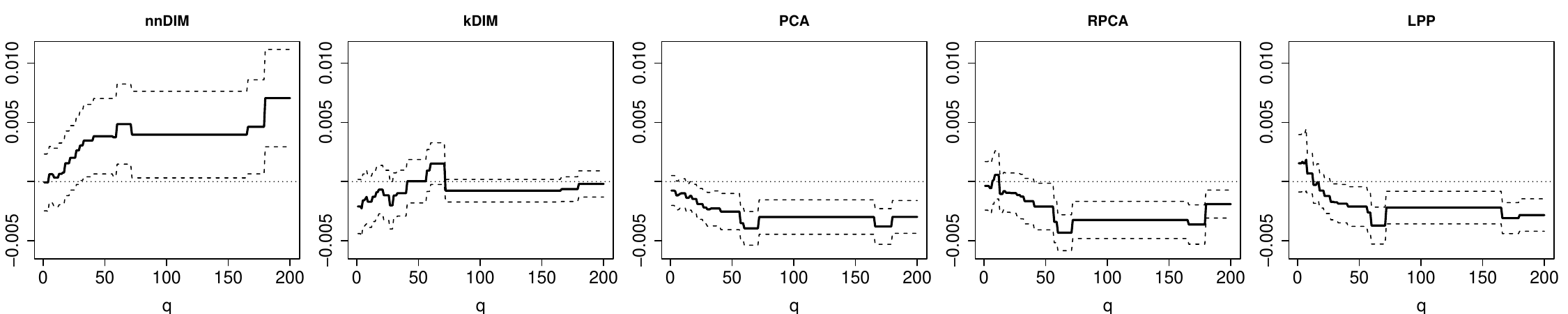}}
    \subfigure[Isolation Forest]{\includegraphics[width=0.99\linewidth]{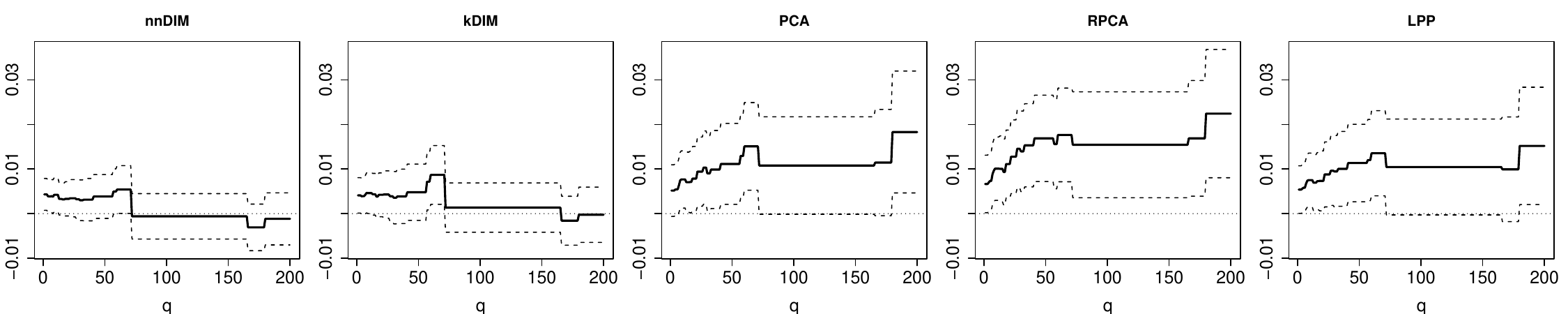}}
    \caption{Single outlier class data sets (type (iii)): Average increases in AUPRC due to dimension reduction, after filtering out data sets of dimension below $q$, for $q \in \{1, 2, ..., 200\}$. Dashed lines correspond with two standard error deviations.}
    \label{fig:outlier_typeiii}
\end{figure}




\section{Discussion and Conclusions}\label{sec:conclusions}

In this paper we introduced an intuitive linear dimension reduction technique based on enhancing the local structure in the data. Emphasising this local structure more around higher density points leads to projections with low entropy relative to overall scale. We demonstrated this theoretically by establishing an asymptotic connection with the Density Information Matrix (DIM), which itself has been connected to the task of Independent Components Analysis (ICA); a classical minimum entropy task.

Finding low entropy projections can be used to aid two of the archetypal unsupervised tasks of cluster analysis and outlier detection, due to both multimodality and long-tailedness being typical characteristics of low entropy distributions. We demonstrated the practical utility of the proposed approach in these contexts, using a large collection of publicly available data sets; where it showed superior overall performance to relevant alternative general purpose linear dimension reduction techniques from the literature.

\section*{Some Additional Proofs and Derivations}

\subsection*{Density Bounded Away from Zero Plus Support Having a Smooth Boundary Implies C3 and C4}

Suppose $\exists \flower > 0$ for which $f_X(\x) \geq \flower \ \forall \x \in Supp(f_X)$, and also that $\exists h_0, q_0 > 0$ for which $\oint_{B_h(\x)\cap Supp(f_X)} 1d\z \geq q_0a_0 h^{p-1}$ for all $\x \in Supp(f_X)$ and $0<h\leq h_0$. Note also that since the gradient and support of $f_X$ are bounded there is also a $\fupper < \infty$ for which $f_X(\x) \leq \fupper \ \forall \x \in \R^p$.

We therefore have, for any $\x \in Supp(f_X)$ and $h_1, h_2 \leq \min\{h_0, 1\}$, that
\begin{align*}
    \frac{f_X(\x)^2 h_1^{p-1}h_2^{p-1}h_2^3}{\oint_{B_{h_1}(\x)}\oint_{B_{h_2}(\x)}f_X(\z)f_X(\w)d\z d\w} &\leq \frac{\fupper^2 h_1^{p-1}h_2^{p-1}}{\oint\limits_{B_{h_1}(\x)\cap Supp(f_X)}\oint\limits_{B_{h_2}(\x) \cap Supp(f_X)}\flower^2 d\z d\w}\\
    &\leq \frac{\fupper^2}{a_0^2q_0^2 \flower^2},
\end{align*}
and hence $C_3$ holds for $l = \frac{\fupper^2}{a_0^2q_0^2 \flower^2}$ and $H = \min\{h_0, 1\}$.

To show that C4 holds, note that a bounded gradient implies that $\exists L > 0$ for which $|f_X(\x)-f_X(\x')| \leq L ||\x -  \x'||$ for all $\x, \x' \in Supp(f_X)$, and we can choose any $\alpha \in (0, 1)$, so that for all $u \in \left(0, \min\{h_0, \frac{\alpha\flower}{L}\}\right)$ and $\gamma \in (0, 1)$,
\begin{align*}
    q_0v_0(f_X(\x)-Lu) u^p \leq \Fdx(u) &\leq v_0(f_X(\x)+Lu) u^p,
\end{align*}
which gives
\begin{align*}
    \frac{\Fdx(\gamma u)}{\Fdx(u)} \leq \frac{f_X(\x)+Lu}{q_0(f_X(\x)-Lu)}\gamma^p \leq \frac{1+\alpha}{q_0(1-\alpha)} \gamma^p,
\end{align*}
with the second inequality holding since $Lu < \alpha \flower \leq \alpha f_X(\x)$. But in addition we have
\begin{align*}
    \Fdx(u)-\Fdx(\gamma u) &\geq q_0v_0(f_X(\x)-Lu)u^p(1-\gamma^p),
\end{align*}
which gives
\begin{align*}
    \frac{\Fdx(\gamma u)}{\Fdx(u)} &= 1-\frac{\Fdx(u)-\Fdx(\gamma u)}{\Fdx(u)} \leq 1-\frac{q_0(1-\alpha)}{1+\alpha}(1-\gamma^p).
\end{align*}
Combining these, we have
\begin{align*}
    \int\limits_1^\infty & \Fdx(u\epsilon^{-1/a})^k d\epsilon = \int\limits_1^{(\frac{1+\alpha}{q_0(1-\alpha)})^{a/p}} \Fdx(u\epsilon^{-1/a})^k d\epsilon + \int\limits_{(\frac{1+\alpha}{q_0(1-\alpha)})^{a/p}}^\infty \Fdx(u\epsilon^{-1/a})^k d\epsilon\\
    \leq & \Fdx(u)^k\left(\int\limits_1^{(\frac{1+\alpha}{q_0(1-\alpha)})^{a/p}}\left(1-\frac{q_0(1-\alpha)}{1+\alpha}(1-\epsilon^{-p/a})\right)^k d\epsilon + \left(\frac{1+\alpha}{q_0(1-\alpha)}\right)^k \int\limits_{(\frac{1+\alpha}{q_0(1-\alpha)})^{a/p}}^\infty \epsilon^{-pk/a}d\epsilon\right).
\end{align*}
Now, it is straightforward to show that both terms in the brackets above are $O(1/k)$, and we are done.

\subsection*{Proof of Lemma~\ref{lem:nnratio}}

We will use the identity $E[\nnnd{\x,n}{k(n)+1}^b/\nnnd{\x,n}{k(n)}^a] = \int_0^\infty P(\nnnd{\x,n}{k(n)+1}^b/\nnnd{\x,n}{k(n)}^a > \epsilon) d\epsilon$. As previously let $\Fdx$ be the distribution function of $||X -  \x||$, i.e. $\Fdx(u) = P(||X - \x|| \leq u)$. We then have, for $\epsilon > 0$ and suppressing the notational dependence of $k(n)$ on $n$, that
    \begin{align*}
        P\left(\frac{\nnnd{\x,n}{k+1}^b}{\nnnd{\x,n}{k}^a} > \epsilon \right) &= P\left(\nnnd{\x,n}{k+1}^{b-a} > \epsilon \right) + P\left(\nnnd{\x,n}{k+1}^{b-a} \leq \epsilon, \nnnd{\x,n}{k} < \frac{\nnnd{\x,n}{k+1}^{b/a}}{\epsilon^{1/a}} \right),
    \end{align*}
    and hence
    \begin{align*}
        E\left[\frac{\nnnd{\x,n}{k+1}^b}{\nnnd{\x,n}{k}^a}\right] =& \int_0^\infty P\left(\nnnd{\x,n}{k+1}^{b-a} > \epsilon \right)d\epsilon\\
        &+ \int_0^\infty P\left(\nnnd{\x,n}{k+1}^{b-a} \leq \epsilon, \nnnd{\x,n}{k} < \frac{\nnnd{\x,n}{k+1}^{b/a}}{\epsilon^{1/a}} \right)d\epsilon\\
        =& E\left[\nnnd{\x,n}{k+1}^{b-a}\right]\\
        &+ \int_0^\infty P\left(\nnnd{\x,n}{k+1}^{b-a} \leq \epsilon, \nnnd{\x,n}{k} < \frac{\nnnd{\x,n}{k+1}^{b/a}}{\epsilon^{1/a}} \right)d\epsilon,
    \end{align*}
    and so it is sufficient to show that the second integral above converges to zero uniformly in $\x$. Towards establishing this, note that, using the properties of order statistics, we have
    \begin{align*}
        P\Bigg(\nnnd{\x,n}{k+1}^{b-a} &\leq \epsilon, \nnnd{\x,n}{k} < \frac{\nnnd{\x,n}{k+1}^{b/a}}{\epsilon^{1/a}} \Bigg)\\
        &= \frac{n!\int\limits_0^{\epsilon^\frac{1}{b-a}}\int\limits_0^{\frac{u^{b/a}}{\epsilon^{1/a}}}\Fdx(z)^{k-1}(1-\Fdx(u))^{n-k-1}d\Fdx(z)d\Fdx(u)}{(k-1)!(n-k-1)!}\\
        &= \frac{n!\int\limits_0^{\epsilon^\frac{1}{b-a}}\Fdx\left(\frac{u^{b/a}}{\epsilon^{1/a}}\right)^{k}(1-\Fdx(u))^{n-k-1}d\Fdx(u)}{k!(n-k-1)!},
    \end{align*}
    and hence, letting $u^* = \min\{1, U\}$ with $U$ as in Condition C4,
    \begin{align*}
        \int_0^\infty\int_0^{\epsilon^\frac{1}{b-a}}& \Fdx\left(\frac{u^{b/a}}{\epsilon^{1/a}}\right)^{k}(1-\Fdx(u))^{n-k-1}d\Fdx(u) d\epsilon\\
        =& \int_0^\infty(1-\Fdx(u))^{n-k-1}\int_{u^{b-a}}^\infty \Fdx\left(\frac{u^{b/a}}{\epsilon^{1/a}}\right)^{k} d\epsilon d \Fdx(u)\\
        =& \int_0^{u^*}(1-\Fdx(u))^{n-k-1}\int_{u^{b-a}}^\infty \Fdx\left(\frac{u^{b/a}}{\epsilon^{1/a}}\right)^{k} d\epsilon d \Fdx(u)\\
        & + \int_{u^*}^\infty(1-\Fdx(u))^{n-k-1}\int_{u^{b-a}}^\infty \Fdx\left(\frac{u^{b/a}}{\epsilon^{1/a}}\right)^{k} d\epsilon d \Fdx(u)\\
        = & \int_0^{u^*}(1-\Fdx(u))^{n-k-1}\Fdx(u)^k\left(\frac{1}{\Fdx(u)^k}\int_{u^{b-a}}^\infty \Fdx\left(\frac{u^{b/a}}{\epsilon^{1/a}}\right)^{k} d\epsilon \right)d \Fdx(u)\\
        & + \int_{u^*}^\infty(1-\Fdx(u))^{n-k-1}\int_{u^{b-a}}^\infty \Fdx\left(\frac{u^{b/a}}{\epsilon^{1/a}}\right)^{k} d\epsilon d \Fdx(u),\\
         \leq & \sup_{0 < u < u^*}\left(\frac{1}{\Fdx(u)^k}\int_{u^{b-a}}^\infty \Fdx\left(\frac{u^{b/a}}{\epsilon^{1/a}}\right)^{k} d\epsilon \right)\int_0^{\infty}(1-\Fdx(u))^{n-k-1}\Fdx(u)^k d \Fdx(u)\\
        & + \int_{u^*}^\infty(1-\Fdx(u))^{n-k-1}\int_{u^{b-a}}^\infty \Fdx\left(\frac{u^{b/a}}{\epsilon^{1/a}}\right)^{k} d\epsilon d \Fdx(u)\\
        = & \frac{(n-k-1)!k!}{n!}\sup_{0 < u < u^*}\left(\frac{1}{\Fdx(u)^k}\int_{u^{b-a}}^\infty \Fdx\left(\frac{u^{b/a}}{\epsilon^{1/a}}\right)^{k} d\epsilon \right)\\
        & + \int_{u^*}^\infty(1-\Fdx(u))^{n-k-1}\int_{u^{b-a}}^\infty \Fdx\left(\frac{u^{b/a}}{\epsilon^{1/a}}\right)^{k} d\epsilon d \Fdx(u)
    \end{align*}
    Now, letting $\fupper$ be an upper bound on $f_X(\x)$, consider that $\Fdx(u^{b/a}\epsilon^{-1/a}) \leq \min\{1, v_0\fupper u^{bp/a}\epsilon^{-p/a}\}$ and so $\int_{u^{b-a}}\Fdx(u^{b/a}/\epsilon^{1/a})^kd\epsilon \leq \int_{0}\min\{1, v_0\fupper u^{bp/a}\epsilon^{-p/a}\}^k d\epsilon = (v_0\fupper)^{p/a}u^b \frac{pk}{pk-a}$. We therefore have
    \begin{align*}
        \int_{u^*}^\infty  (1-\Fdx(u))^{n-k-1}&\int_{u^{b-a}}^\infty \Fdx\left(\frac{u^{b/a}}{\epsilon^{1/a}}\right)^{k} d\epsilon d \Fdx(u)\\
        &\leq (v_0\fupper)^{p/a} \frac{pk}{pk-a}\int_{u^*}^\infty u^b(1-\Fdx(u))^{n-k-1} d \Fdx(u)\\
        &\leq (v_0\fupper)^{p/a} (2M)^b \frac{pk}{pk-a}\int_{u^*}^\infty (1-\Fdx(u))^{n-k-1} d \Fdx(u)\\
        &= (v_0\fupper)^{p/a} (2M)^b \frac{pk}{(pk-a)(n-k)}(1-\Fdx(u^*))^{n-k},
    \end{align*}
    where $M$ is as in Condition C1. Then, since the closure of the support of $f_X$, $\overline{Supp(f_X)}$, is compact, it must be that $\exists \x^* \in \overline{Supp(f_X)}$ for which $0 < F_{(\x^*)}(u^*) \leq \min_{\x \in Supp(f_X)}\Fdx(u^*)$, and as a result we have
    \begin{align*}
        \sup_{\x \in Supp(f_X)} \int_{u^*}^\infty  (1-\Fdx(u))^{n-k-1}&\int_{u^{b-a}}^\infty \Fdx\left(\frac{u^{b/a}}{\epsilon^{1/a}}\right)^{k} d\epsilon d \Fdx(u)\\
        &\leq C\frac{1}{n-k}\left(1-F_{(\x^*)}(u^*)\right)^{n-k},
    \end{align*}
    for some constant $C$, independent of $n, k$.

    Next observe that, for $u < u^* \leq 1$ and using the change of variables $v = u^{b/a}/\epsilon^{1/a}$, we have
    \begin{align*}
        \int_{u^{b-a}}^\infty \Fdx(u^{b/a}\epsilon^{1/a})d\epsilon &=  a u^b \int_0^u v^{-(a+1)}\Fdx(v) dv\\
        &\leq au^a \int_0^u v^{-(a+1)}\Fdx(v) dv = \int_{1}^\infty \Fdx(u\epsilon^{1/a})d\epsilon,
    \end{align*}
    and hence
    \begin{align*}
        \sup_{0 < u < u^*}\left(\frac{1}{\Fdx(u)^k}\int_{u^{b-a}}^\infty \Fdx\left(\frac{u^{b/a}}{\epsilon^{1/a}}\right)^{k} d\epsilon \right) \leq \sup_{0 < u < u^*}\left(\frac{1}{\Fdx(u)^k}\int_{1}^\infty \Fdx\left(\frac{u}{\epsilon^{1/a}}\right)^{k} d\epsilon \right).
    \end{align*}
    Combining this we have, for all $\x \in Supp(f_X)$
    \begin{align*}
        E\left[\frac{\nnnd{\x,n}{k(n)+1}^b}{\nnnd{\x,n}{k(n)}^a}\right]& - E\left[\nnnd{\x,n}{k(n)+1}^{b-a}\right] \leq \frac{n!}{k!(n-k-1)!}\Bigg(C\frac{1}{n-k}\left(1-F_{(\x^*)}(u^*)\right)^{n-k}\\
        & + \frac{(n-k-1)!k!}{n!}\sup_{0 < u < u^*}\left(\frac{1}{\Fdx(u)^k}\int_{1}^\infty \Fdx\left(\frac{u}{\epsilon^{1/a}}\right)^{k} d\epsilon \right)\Bigg)\\
        &= C{n \choose k}\left(1-F_{(\x^*)}(u^*)\right)^{n-k} + \sup_{0 < u < u^*}\left(\frac{1}{\Fdx(u)^k}\int_{1}^\infty \Fdx\left(\frac{u}{\epsilon^{1/a}}\right)^{k} d\epsilon \right),
    \end{align*}
    where the first term is independent of $\x$ and clearly converges to zero as $n \to \infty$ provided $k/n \to 0$, and the second term converges to zero uniformly in $\x$ by Condition C4 since $k \to \infty$ as $n \to \infty$.

\end{document}